\pdfoutput=1
\documentclass[11pt]{article}

\usepackage[letterpaper,margin=1in]{geometry}

\usepackage{lipsum}

\usepackage{caption}

\usepackage[utf8]{inputenc} 
\usepackage[T1]{fontenc}    

\usepackage{url}            
\usepackage{booktabs}       
\usepackage{nicefrac}       
\usepackage{microtype}      
\usepackage{enumitem}
\usepackage{dsfont}
\usepackage{graphicx}
\usepackage{multicol}
\usepackage{xcolor}
\usepackage{tikz}

\usepackage{amsmath,amsfonts,bm}

\newcommand{\bx}{\mathbf{x}}
\newcommand{\bb}{\mathbf{b}}
\newcommand{\bw}{\mathbf w}
\newcommand{\ba}{\mathbf a}
\DeclareMathOperator{\maj}{Maj}

\renewcommand{\ba}{\textbf{a}}

\def\eqref#1{equation~\ref{#1}}

\def\1{\bm{1}}

\def\eps{{\epsilon}}

\DeclareMathAlphabet{\mathsfit}{\encodingdefault}{\sfdefault}{m}{sl}
\SetMathAlphabet{\mathsfit}{bold}{\encodingdefault}{\sfdefault}{bx}{n}

\newcommand{\E}{\mathbb{E}}

\newcommand{\R}{\mathbb{R}}

\usepackage{amsthm,amsfonts,amsmath,amssymb,epsfig,color,float,graphicx,verbatim, enumitem, subfigure}

\usepackage{algorithm}
\usepackage[noend]{algpseudocode}

\usepackage{mathtools}
\usepackage{bbm}

\usepackage{thm-restate}
\usepackage{turnstile}

\usepackage[colorlinks,citecolor=blue,linkcolor=magenta,bookmarks=true,hypertexnames=false]{hyperref}
\usepackage[nameinlink]{cleveref}

\crefname{equation}{equation}{equations}
\crefname{lemma}{lemma}{lemmata}
\crefname{claim}{claim}{claims}
\crefname{theorem}{theorem}{theorems}
\crefname{proposition}{proposition}{propositions}
\crefname{corollary}{corollary}{corollaries}
\crefname{claim}{claim}{claims}
\crefname{remark}{remark}{remarks}
\crefname{definition}{definition}{definitions}
\crefname{fact}{fact}{facts}
\crefname{question}{question}{questions}
\crefname{condition}{condition}{conditions}
\crefname{algorithm}{algorithm}{algorithms}
\crefname{assumption}{assumption}{assumptions}
\crefname{problem}{problem}{problems}

\newtheorem{theorem}{Theorem}[section]
\newtheorem{lemma}[theorem]{Lemma}

\newtheorem{corollary}[theorem]{Corollary}
\newtheorem{claim}[theorem]{Claim}

\newtheorem{definition}[theorem]{Definition}

\theoremstyle{definition}

\newtheorem{problem}[theorem]{Problem}

\newcommand{\cons}{\text{\cons}}

\newcommand{\poly}{\mathrm{poly}}

\def\R{\mathbb R}

\def\Z{\mathbb Z}

\newcommand{\citep}{\cite}
\newcommand{\citet}{\cite}

\newcommand{\bW}{\vec{W}}

\newcommand{\hide}[1]{}

\let\vec\mathbf

\def\colorful{1}

\ifnum\colorful=0
\newcommand{\new}[1]{{\color{red} #1}}

\else
\newcommand{\new}[1]{{#1}}

\fi

\allowdisplaybreaks

\title{Efficient Knowledge Distillation via Curriculum Extraction}

\author{
Shivam Gupta\\
University of Texas at Austin
\thanks{Now at Google Research.}\\
{\tt shivamgupta@utexas.edu }\\
\and
Sushrut Karmalkar\\
Microsoft Research, Cambridge
\thanks{Work done prior to joining Microsoft Research.}\\
{\tt skarmalkar@microsoft.com}\\
}

\begin{document}
\maketitle	
	\begin{abstract}
Knowledge distillation is a technique used to train a small student network using the output generated by a large teacher network, and has many empirical advantages~\citep{Hinton2015DistillingTK}. While the standard one-shot approach to distillation only uses the output of the final teacher network, recent work~\citep{panigrahi2024progressive} has shown that using intermediate checkpoints from the teacher's training process as an implicit ``curriculum'' for progressive distillation can significantly speed up training. However, such schemes require storing these checkpoints, and often require careful selection of the intermediate checkpoints to train on, which can be impractical for large-scale training.

In this paper, we show that a curriculum can be \emph{extracted} from just the fully trained teacher network, and that this extracted curriculum can give similar efficiency benefits to those of progressive distillation. 
Our extraction scheme is natural; we use a random projection of the hidden representations of the teacher network to progressively train the student network, before training using the output of the full network. 
We show that our scheme significantly outperforms one-shot distillation and achieves a performance similar to that of progressive distillation for learning sparse parities with two-layer networks, and provide theoretical guarantees for this setting. 
Additionally, we show that our method outperforms one-shot distillation even when using transformer-based architectures, both for sparse-parity learning, and language modeling tasks.
\end{abstract}

\setcounter{page}{0}

\thispagestyle{empty}

\newpage

	\section{Introduction}
In the era of large-scale models, as the cost of training state-of-the-art models increases substantially with each passing year, leveraging compute effectively for training and inference has become increasingly important. Knowledge distillation~\citep{Hinton2015DistillingTK} is one popular technique that is commonly used to reduce the amount of compute necessary for inference, by training a small \emph{student} network to mimic the output of a large teacher network. Indeed, several state-of-the-art language models are distilled versions of larger models~\citep{deepseekai2025deepseekr1incentivizingreasoningcapability,abdin2024phi4technicalreport,geminiteam2024gemini15unlockingmultimodal,openai-o1mini}.

\begin{figure}[t]
    \centering
    \begin{minipage}[b]{0.4\linewidth}
        \centering
        \hspace{-.8cm}
        \scalebox{0.5}{\usetikzlibrary{positioning, backgrounds, shapes.geometric, calc} 

\begin{tikzpicture}[
    node distance=1cm,
    rect/.style={
        draw,
        minimum width=3cm,
        minimum height=1cm,
        fill=blue!30,
        rounded corners
    },
    rectstudent/.style={
        draw,
        minimum width=2.5cm,
        minimum height=0.5cm,
        fill=red!30,
        rounded corners
    },
    background/.style={
        fill=gray!30,
        fill opacity=0.5,
        rounded corners
    },
    trapezstudent/.style={
        draw,
        fill=red!30,
        trapezium, 
        trapezium stretches=true,
        trapezium left angle=75,
        trapezium right angle=75,
        minimum width=2.5cm,
        minimum height=0.25cm
    },
    trapez/.style={
        draw,
        fill=blue!30,
        trapezium, 
        trapezium stretches=true,
        trapezium left angle=75,
        trapezium right angle=75,
        minimum width=3cm,
        minimum height=0.5cm
    },
    trapez dashed/.style={
        trapez,
        fill=gray!30,
        draw=black,
        dashed,
    }
]

\node[rect] (rect1) {};
\node[rect, below=of rect1] (rect2) {};
\node[rect, below=of rect2] (rect3) {};

\node[trapez, above=of rect1, yshift=0.3cm] (trap) {};

\node[trapez dashed, right=of rect3, yshift=1.5cm] (rand1) {}; 
\node[trapez dashed, right=of rect2, yshift=1.5cm] (rand2) {}; 
\node[trapez dashed, right=of rect1, yshift=1.6cm] (rand3) {}; 

\node[rectstudent, right=of rand1, yshift=-0.8cm] (rectstudent1) {};
\node[rectstudent, right=of rand2, yshift=-0.8cm] (rectstudent2) {};
\node[rectstudent, right=of rand3, yshift=-0.8cm] (rectstudent3) {};

\node[trapezstudent, above=of rectstudent3] (trapstudent) {};
\draw[background] 
    ([shift={(-1.3cm,1cm)}]trapstudent.north west) 
    rectangle 
    ([shift={(0.2cm,-1cm)}]rectstudent1.south east);
\path let \p1 = ([shift={(-1.5cm,1cm)}]trapstudent.north west), 
          \p2 = ([shift={(0.3cm,-1cm)}]rectstudent1.south east) in 
          node[font=\Large] at ({(\x1 + \x2)/2}, \y1 - 0.3cm) {Student};
\draw[background] 
    ([shift={(-1.5cm,1cm)}]trap.north west) 
    rectangle 
    ([shift={(0.3cm,-1cm)}]rect3.south east);

\path let \p1 = ([shift={(-1.5cm,1cm)}]trap.north west), 
          \p2 = ([shift={(0.3cm,-1cm)}]rect3.south east) in 
          node[font=\Large] at ({(\x1 + \x2)/2}, \y1 - 0.3cm) {Teacher};

\draw[<->, dashed, thick] (rand1.north) to [bend left=30] node[midway, above] {\large 1} (rectstudent1.north);
\draw[<->, dashed, thick] (rand2.north) to [bend left=30] node[midway, above] {\large 2}(rectstudent2.north);
\draw[<->, dashed, thick] (rand3.north) to [bend left=30] node[midway, above] {\large 3} (rectstudent3.north);
\draw[<->, dashed, thick] (trap.north) to [bend left=30] node[midway, above] {\large 4} (trapstudent.north);

\draw[->] (rect2) -- (rect1);
\draw[->] (rect3) -- (rect2);
\draw[->] (rect1) -- (trap);

\draw[->] (rectstudent1) -- (rectstudent2);
\draw[->] (rectstudent2) -- (rectstudent3);
\draw[->] (rectstudent3) -- (trapstudent);

\path (rect3) -- (rect2) coordinate[midway] (midpoint1);
\path (rect2) -- (rect1) coordinate[midway] (midpoint2);
\path (rect1) -- (trap) coordinate[midway] (midpoint3);

\draw[->] (midpoint1) -| (rand1.south) -- (rand1.south);
\draw[->] (midpoint2) -| (rand2.south) -- (rand2.south);
\draw[->] (midpoint3) -| (rand3.south) -- (rand3.south);

\node[background, anchor=north west, minimum width=4.8cm, minimum height=0.8cm] (legend) at ([shift={(-3.2cm,-1.3cm)}]rectstudent1.south) {}; 

\node[trapez dashed, minimum width=1.2cm, minimum height=0.4cm] at ([xshift=0.8cm]legend.west) {};
\node[right=1.4cm of legend.west] {Random projection};

%
%

\end{tikzpicture}}
            \caption{%
    Our curriculum extraction method trains the student model in a layer-wise fashion. Student layers are sequentially aligned to a random projection of the corresponding teacher layer’s hidden representation using the Mean Squared Error (MSE). After aligning layers, the student is trained on the teacher’s output logits via the KL Divergence loss.}
            \label{fig:curriculum_extraction_diagram}
    \end{minipage}
    \hspace{1cm}
        \begin{minipage}[b]{0.4\linewidth}
        \centering
        \includegraphics[width=1\textwidth]{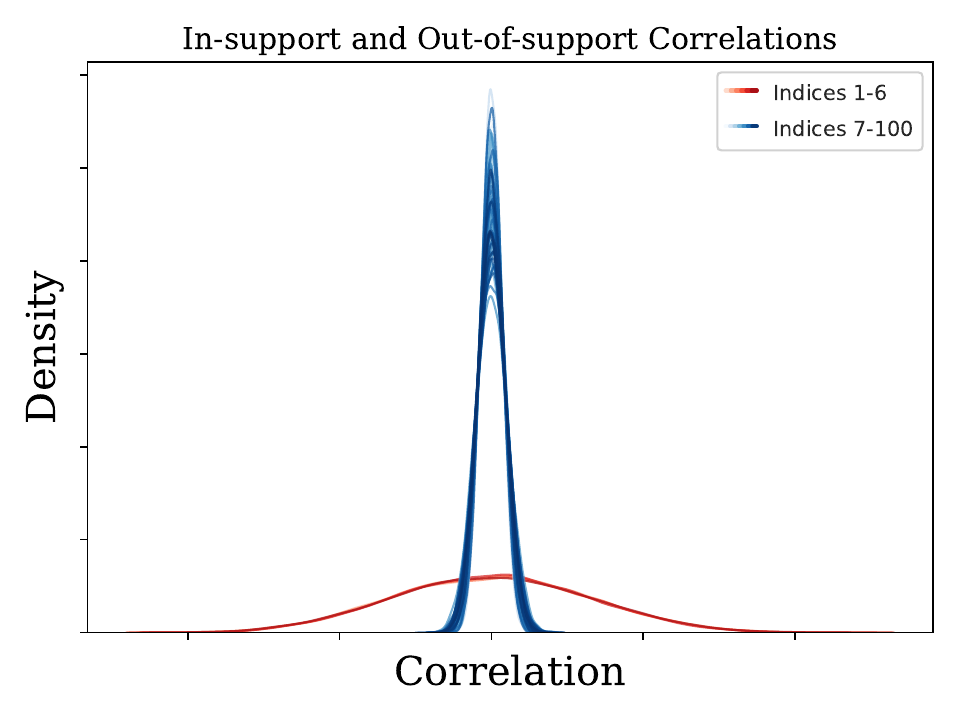}
        \caption{\textbf{In-support and out-of-support correlations.} 
     A two-layer MLP trained on 100-dimensional 6-sparse parity data exhibits distinct in-support (red) and out-of-support (blue) correlations of $(Af_t^{(1)})(\bx)$ with $x_j$ for the random projection $A \in \R^{1 \times m_t}$. When $j$ is in the support, the correlations show significantly larger standard deviations compared to when $j$ is outside the support. }
        \label{fig:correlation_plot}
    \end{minipage}

\end{figure}

\begin{figure*}[t]
    \centering
    \begin{minipage}[t]{0.31\linewidth}
        \centering
            \hspace*{-2mm} 
        \includegraphics[width=1.12\textwidth]{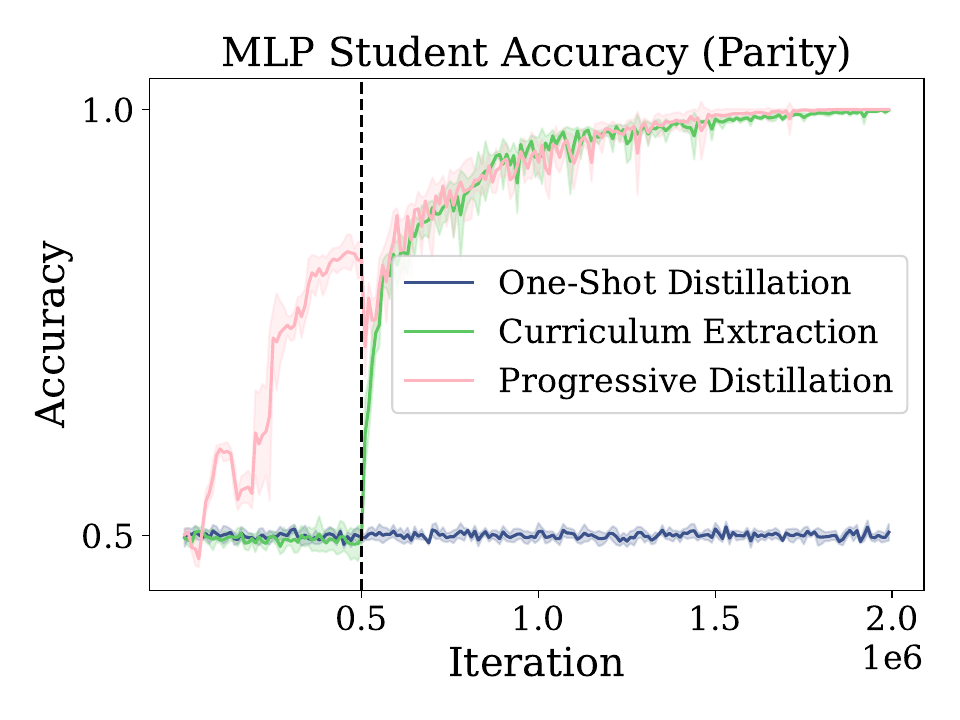}
        \centering\scriptsize\textbf{(a)} 
        \label{fig:mlp_student_accuracy}
    \end{minipage}
    \hfill
    \begin{minipage}[t]{0.31\linewidth}
    \centering
        \hspace*{-3mm} 
    \includegraphics[width=1.12\textwidth]{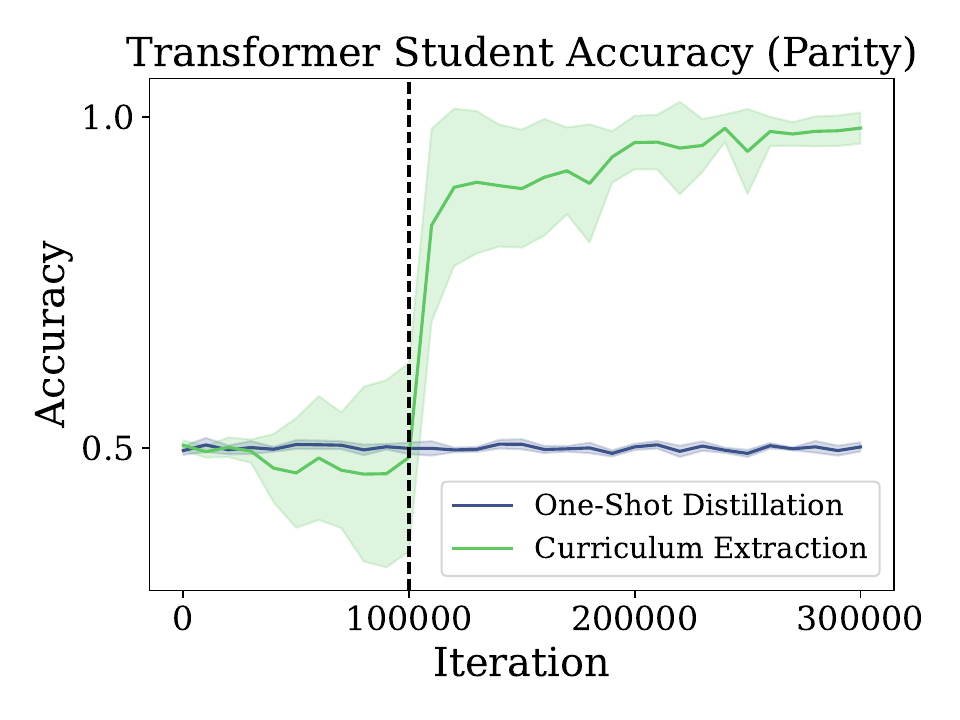}
        \centering\scriptsize\textbf{(b)} 
    \label{fig:transformer_student_accuracy}
    \end{minipage}
    \hfill
        \begin{minipage}[t]{0.31\linewidth}
    \centering
    \hspace*{-3mm} 
    \includegraphics[width=1.12\textwidth]{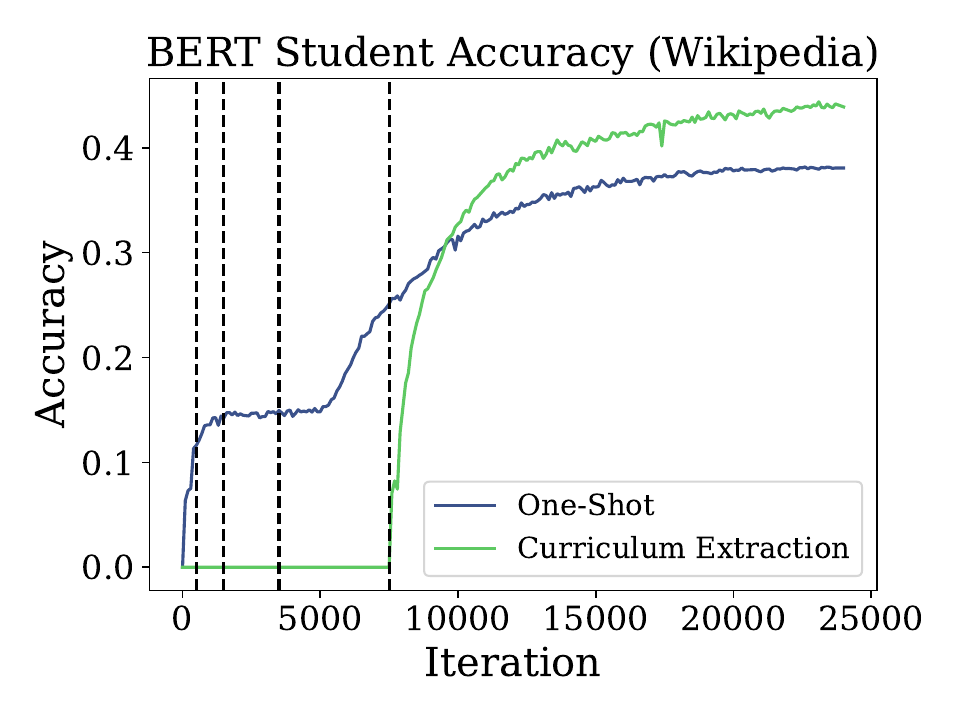}
        \centering\scriptsize\textbf{(c)} 
    \label{fig:wiki}
    \end{minipage}
    \caption{%
    \textbf{Comparing Curriculum Extraction and One-Shot Distillation.}
     We show three tasks for which curriculum extraction outperforms one-shot distillation:
(a) A two-layer MLP trained on 100-dimensional 6-sparse parity, with a teacher hidden dimension of 50k and a student hidden dimension of 100.
(b) A transformer trained on 100-dimensional 6-sparse parity, using 256-dimensional embeddings, where the teacher has 32 attention heads and the student has 4.
(c) A BERT-large model fine-tuned on the Wikipedia dataset, with the teacher using 768-dimensional embeddings, 12 attention heads, and 12 transformer blocks, while the student reduces embeddings to 256 dimensions and attention heads to 4. The dashed vertical lines indicate the iterations where the layer being distilled is changed in the case of curriculum extraction, and a change in teacher checkpoint in the case of progressive distillation.}
    \label{fig:student_accuracies}
\end{figure*}

Despite the adoption of distillation for language model training, prior work~\citep{panigrahi2024progressive, anil2020largescaledistributedneural,mirzadeh2019improvedknowledgedistillationteacher} has shown that just using the output of the fully-trained teacher network to train the student can result in poor performance relative to the teacher (a ``teacher-student gap'' in performance), and that \emph{progressive distillation} can significantly improve the performance of the student. \citet{panigrahi2024progressive} in particular offers an explanation for this phenomenon -- the intermediate checkpoints during the training of the teacher network act as an implicit \emph{curriculum} for the training of the student network, with earlier checkpoints emphasizing simpler patterns (e.g., local syntax in the case of language models), and later checkpoints capturing complex abstractions (e.g., long-range semantics). Please see \Cref{sec:related_work} for a more detailed overview of related work. 

While progressive distillation offers significant efficiency advantages over one-shot distillation, it requires storing frequent checkpoints during the training of the teacher, which can be prohibitive for modern LLMs. Deciding on which checkpoints to make use of to train the student is often unclear; such checkpoints are found by extensive experimentation in the works listed above, which can be impractical. Moreover, in many cases, one lacks access to intermediate checkpoints during training, even for open-source models~\citep{jiang2023mistral7b,deepseekai2025deepseekr1incentivizingreasoningcapability,meta2024llama}, making progressive distillation impossible.

This raises a natural question -- can we design a scheme that maintains the advantages of progressive distillation without suffering from its drawbacks? Specifically, can we leverage the final fully-trained teacher model more effectively to train the student model efficiently?

\paragraph{Curriculum Extraction.} We propose a scheme to \emph{extract} a curriculum from the fully-trained teacher network. Our key insight is that the layer-wise hierarchy of a fully trained network naturally encodes a progression from simple to complex features. To operationalize this, we train the student's hidden layers sequentially on \emph{random projections} of the teacher’s hidden layers, starting from shallow (layer $l$, say) to deep (the final layer $L$), before training the full student network on the output of the full teacher network. See Figure~\ref{fig:curriculum_extraction_diagram} for a visual description of our extraction scheme.

 By progressively training on projections from shallower to deeper layers, the student learns incrementally—mirroring the coarse-to-fine learning in progressive distillation, without having to store intermediate checkpoints. Beyond circumventing checkpoint storage, our approach can potentially be applied to efficiently distill from open-source models (e.g., Llama~\citep{meta2024llama}, Mistral~\citep{jiang2023mistral7b}, and Deepseek models~\citep{deepseekai2025deepseekr1incentivizingreasoningcapability}), where only the final model is available. In addition, our method is computationally cheaper than one-shot or progressive distillation per training iteration  --  during early stages, only a subset of student and teacher layers are active, reducing memory and FLOPs per iteration.  

\subsection{Our Results}

\paragraph{Sparse Parity Learning.}
We show that our curriculum extraction scheme is \emph{significantly} more efficient than one-shot distillation for the task of learning sparse parities using a two-layer MLP, and provide a theoretical analysis for this setting. See Section~\ref{sec:sparse_parities} for a formal description of sparse parity learning and two-layer MLPs. Here, we state an informal version of our main theorem, with the formal theorem stated in \Cref{sec:main-thm}.

\begin{theorem}[Main, Informal]
\label{thm:main_informal}
Let $d \geq \tilde \Omega(k^4)$. 
Consider learning $d$-dimensional $k$-sparse parity with a student model of size $\tilde{\Theta}(2^{O(k)})$, where $\tilde{O}, \tilde{\Theta}$ hides polylog factors in $d,k$. Suppose the teacher 
has a loss $\epsilon/C$ for some sufficiently large constant $C > 0$ and sufficiently small $\epsilon > 0$. Then, the total sample complexity needed for the student to reach $\epsilon$-loss using curriculum extraction based on random projection is: $\tilde{\Theta}\left(2^{O(k)} \poly(d) \epsilon^{-2}\right)$. However, one-shot distillation requires at least $\Omega\left(d^{k-1} \epsilon^{-2}\right)$ samples.\end{theorem}

 Thus, one-shot distillation \emph{requires} $\Omega(d^{O(k)})$ samples to learn $k$-sparse parity over a $d$-dimensional domain, while our curriculum extraction scheme can learn using only $\tilde O(2^{O(k)} \poly(d))$ samples. 
We show in Figure~\ref{fig:student_accuracies} (a) that our curriculum extraction scheme significantly outperforms one-shot distillation empirically, as predicted by our theory -- our scheme succeeds in learning, while one-shot distillation fails after training using $2 \cdot 10^6$ samples. Furthermore, it has similar performance as progressive distillation for a carefully chosen checkpoint -- we choose the checkpoint during training of the teacher network whose output is most correlated with the support of the parity function, as proposed by~\citet{panigrahi2024efficient}.

We also show empirically that our scheme continues to outperform one-shot distillation when using a transformer-based architecture for learning sparse parities in Figure~\ref{fig:student_accuracies} (b).

\paragraph{Masked Language Modeling (BERT).}

In addition to learning sparse parities, we empirically study our curriculum extraction scheme for language modeling, focusing on BERT-style \emph{masked language modeling}. We study two settings with different kinds of data: (i) Synthetic data generated by a Probabilistic Context-free Grammar (PCFG), and (ii) Real-world language data from Wikipedia.

In the case of PCFGs, we show that our scheme outperforms one-shot distillation, both in terms of computational efficiency (number of FLOPs), and in terms of sample efficiency (number of iterations), in Figures~\ref{fig:pcfg_bert} (a) and (b). We also show that our curriculum scheme outperforms just using the final hidden representation of the teacher to distill before distilling using the full network; this suggests that the efficiency benefits of our scheme do indeed come from the fact that the layers of the teacher network implicitly act as a curriculum, rather than merely from the increased dimensionality of the distilled features.

For Wikipedia data, we show in Figure~\ref{fig:student_accuracies} (c) that curriculum extraction has a significant accuracy advantage over one-shot distillation when training for $24 \cdot 10^3$ iterations with a batch size of $128$ -- for extraction, we use $4$ intermediate layers to distill, before distilling with the full teacher network.

\section{Related Work}
\label{sec:related_work}

\paragraph{Knowledge Distillation}  
Knowledge distillation (KD), pioneered by \cite{Hinton2015DistillingTK}, transfers knowledge from computationally expensive teacher models to lightweight students by aligning output distributions. This paradigm has been widely adopted in modern language models for inference cost reduction. Examples of models trained via distillation include ChatGPT O1-mini~\citep{openai-o1mini}, Gemini Flash~\citep{geminiteam2024gemini15unlockingmultimodal}, and the Phi series of models~\citep{abdin2024phi4technicalreport}. Despite the success of distillation, there have been a number of works~\citep{mirzadeh2019improvedknowledgedistillationteacher,cho,supervision_complexity,panigrahi2024progressive} that have observed that simply distilling using the output of the fully trained teacher network can be suboptimal, resulting in a ``teacher-student gap'' in capabilities. To circumvent this gap, these works have proposed \emph{progressive distillation}, a technique that trains the student using intermediate checkpoints during the teacher training progressively, before training on the output of the fully-trained teacher. \citet{panigrahi2024progressive} proposes that the intermediate checkpoints act as an implicit \emph{curriculum}, allowing the student to learn simpler functions before moving to the final complex one, and shows theoretical and empirical evidence to support these claims.

 Despite its promise, progressive distillation faces some challenges -- (1) storing intermediate checkpoints for large models incurs prohibitive costs, (2) finding an effective checkpoint schedule to use is done heuristically \citep{panigrahi2024progressive}, requiring costly trial-and-error, (3) Most models, including open-source ones (e.g., Llama 3~\citep{meta2024llama}, Mistral~\citep{jiang2023mistral7b}) only release final weights, making progressive distillation impossible.

 
Our method builds a curriculum using just the fully-trained teacher network, and thus, avoids the shortcomings or progressive distillation. While our curriculum extraction method is related to layer-wise distillation methods proposed in the literature~\citep{aguilar2020knowledge, liang2023less, jiao2020tinybert,Sun2019PatientKD}, these methods typically require the teacher and student to have the same embedding dimension. In contrast, our method uses a \emph{random projection} to embed the teacher's hidden representation into the student's embedding dimension, allowing us to accommodate differing embedding dimensions. Our method also crucially relies on the stage-wise curriculum extraction of each layer individually, while the aforementioned works typically distill teacher layers simultaneously with the final output. Moreover, all of  works are the prior works are heuristic in nature; we provide rigorous guarantees showing the correctness of our method, albeit in a stylized setting. 

\paragraph{Curriculum Learning}  
Curriculum learning, formalized by \citet{bengio2009curriculum}, structures training data by difficulty to improve learning efficiency. Early NLP work relied on handcrafted curricula \citep{kocmi2017curriculum}, while modern approaches automate this process through self-paced learning \citep{kumar2010self}. Recent advances, such as those by \citep{panigrahi2024efficient}, demonstrate that stagewise pretraining via incremental subnetworks achieves computational efficiency without compromising model quality, further supporting the benefits of progressive learning paradigms. Additionally, studies on BERT reveal that its layers encode linguistic hierarchies \citep{tenney2019bert,hewitt2019structural}, suggesting that intermediate representations can scaffold learning—though this typically requires access to intermediate checkpoints \citep{shwartz2017opening,voita2019analyzing}.

\section{Curriculum Extraction}

We now describe our curriculum extraction scheme formally.

\begin{definition}[Curriculum Extraction Scheme]
\label{def:scheme}
Given a pre-trained teacher network $~T$ and a student network $S$, both having the same number of layers $L$, let $T_i$ and $S_i$ denote the network up to layer $i$ ($T_i:\R^{d} \rightarrow \R^{m_i}$ and $S_i:\R^{d} \rightarrow \R^{n_i}$).
Suppose also that for some $\ell \in [0, L)$ we are given a sequence $\{t_\ell, \dots, t_L\}$ such that $t_i \in \Z$ indicates the number of iterations we train $S_i$ for.
The Layer-Wise Curriculum Extraction Scheme proceeds as follows:

\begin{enumerate}
    \item \textbf{Initialization:} 
    Initialize the student network $S$ with random weights.
    \item \textbf{Layer-Wise Training:}
     For each $i \in [\ell, L-1]$, such that $t_i > 0$:
        \begin{enumerate}
            \item Define a random projection matrix $P_i : \R^{m_i} \rightarrow \R^{n_i}$ for layer $i$.
            \item Train \(S_i\) for \(t_i\) iterations to reduce the MSE loss between \(S_i(x)\) and \(P_i(T_i(x))\):
            $\mathcal{L}_i = \frac{1}{T} \sum_{t=1}^{T} \| S_i(x^{(t)}) - P_i(T_i(x^{(t)})) \|^2_2$
            where $T$ is the number of training samples.
        \end{enumerate}
    \item 
    Train the entire student network $S$ for $t_L$ iterations to reduce the KL-divergence loss between $S(x)$ and $T(x)$.
\end{enumerate}
\end{definition}

\section{Learning Sparse Parities via Curriculum Extraction}
\label{sec:sparse_parities}

To demonstrate the effectiveness of our curriculum extraction scheme, we study its performance in learning sparse parities. 
We compare our curriculum extraction scheme to one-shot distillation, where the student is trained directly on samples generated by the teacher. 



\subsection{Preliminaries}
\label{subsec:prelims}

For our arguments in this section, we will assume WLOG that the support of the unknown parity is $S = [k]$. 
We will learn two-layer MLP networks of the form 
$f(\bx) := \ba \cdot \sigma(\bW \bx + \bb) = \sum_{i=1}^m a_i \sigma(\bw_i \cdot \bx + b_i)$. 
where $\bx \in \R^d; \bW \in \R^{d \times m}$; $\bb, \ba \in \R^m$, and $\sigma(t) := \max(0, t)$ is applied coordinate-wise when applied to a vector. 
We will denote the student network by $f_s$ and the teacher network by $f_t$ with hidden dimensions $m_s$ and $m_t$ respectively.
In general $m_s \leq m_t$. 
Let $\ell_f(\bx, y) := \max(0, 1-f(\bx)y)$ be the hinge loss. 
Our main task will be to find the best fitting two-layer MLP to an unknown sparse parity function. 

\begin{problem}[Learning Sparse Parities]
\label{prob:sparse_parities}
    Let \( S \subset [d] \) with \( |S| = k \) and \( k < d \) denote the support of our unknown sparse parity. 
    For \( \bx \in \{\pm 1\}^d \), we define $\chi_S(\bx) := \prod_{i \in S} x_i$ to be a \emph{sparse parity} supported on $S$.  
    Given a tolerance \(\epsilon \in \R\) and $n$ samples \(\{(\bx_i, \chi_S(\bx_i)) \mid \bx_i \sim_{u.a.r} \{\pm 1\}^d \text{ for } i \in [0, n]\}\) for an unknown support $S$, 
    the task of learning a sparse parity function using a two-layer MLP, is to find a two-layer MLP that achieves loss $\E_{(\bx, y)} [\ell_f(\bx,y)] \leq \epsilon$. 
\end{problem}

For our theoretical analysis, our training setup differs slightly from \Cref{def:scheme} when it comes to our losses -- we use a regularized version of the hinge loss to train the bottom layer of the teacher, the hinge loss to train the top layer of the teacher as well as the student, and a regularized version of a correlation-based distillation loss (defined below) to train the student's hidden layer.

\paragraph{Initialization}
Prior to training, we will initialize the network using the following symmetric initialization from \citet{barak2022hidden}.

\begin{definition}[Symmetric Initialization]
\label{def:sym_initialization}
Let $f(\bx) := \sum_{i=1}^m a_i \sigma(\bw_i \cdot \bx + b_i)$ be a two-layer MLP with input dimension $d$ and hidden dimension $m$. For each 
$1 \leq i \leq m/2,$ we initialize the parameters $\{\bw_i\}_{i=1}^m$, $\{b_i\}_{i=1}^m$ and $\{a_i\}_{i=1}^m$ as follows: $\bw_i \sim U(\{\pm 1\}^d),$
$b_i \sim U\left(\left\{-1 + \frac{1}{k}, \cdots, 1 - \frac{1}{k}\right\}\right), 
\quad a_i \sim U\left(\left\{\frac{\pm 1}{m}\right\}\right)$, $m/2 < i \leq m$ are set to $\bw_i = -\bw_{i -m/2}, \quad b_i = b_{i -m/2}, \quad a_i = -a_{i -m/2}$.
\end{definition}

\paragraph{Training Algorithms:}
We train the teacher network $f_t$ by minimizing the hinge loss in two stages. 
In the first stage of training, we freeze the top layer weights $(\ba)$ and train the network with a regularized version of $\ell_{f_t}(\bx, y)$, given by $\ell_{f_t}(\bx, y) - \lambda \|\bW\|^2$ updating only $\bW$.
In the second stage of training, we freeze the bottom layer weights and biases $\bW, \bb$ and only update $\ba$. 

For the student network $f_s$, we define a \emph{distillation loss} instead.
Let $f_t^{(1)} := \sigma(\bW \bx + \bb) : \R^d \rightarrow \R^{m_t}$ denote the output of the first layer of the teacher and suppose $A \in \R^{m_t \times m_s}$ is a random symmetric projection which mimics the initialization, i.e. for $i \leq m_t/2$, each \new{$A_{ij} \sim U(\{\pm 1/m_t\})$} and for $i > m_t/2$, 
$A_{ij} = -A_{(i-m/2)~j}$.
In the first stage (where we only update the bottom layer weights), the first layer of the student (i.e. $f_s^{(1)}(\bx) := \sigma(\bW_s \bx + \bb_s)$) is trained using the a similarly regularized version of the following distillation loss: 
$\ell_{DL}(\bx, f_s^{(1)}, Af_t^{(1)}) = \new{-f^{(1)}_s(\bx) \cdot (Af^{(1)}_t(\bx))}$, i.e. we use the loss $\ell_{DL}(\bx, f_s^{(1)}, Af_t^{(1)}) - \lambda \|\bW_s\|^2$ for a carefully chosen value of $\lambda$.

The second layer of the student ($\ba_s$) is then trained using the standard hinge loss.

\subsection{Theoretical Results}
\label{sec:main-thm}
We prove that, compared to one-shot distillation—where the student needs at least \(\Omega\left(d^{k-1}\right)\) samples to learn the unknown sparse parity—our curriculum extraction method reduces this requirement to \(\tilde{O}(2^{O(k)}\poly(k,d))\).
We state this formally below:
\begin{theorem}[Curriculum Extraction Requires Fewer Samples]
\label{thm:main}
    Let $m_s$ and $m_t$ denote the student and teacher hidden dimensions, with $m_t \geq m_s$ and $d \geq\tilde \Omega(k^4)$.
    Suppose the teacher model 
    has been trained to learn a $d$-dimensional $k$-sparse parity with the 2-stage training algorithm in \Cref{alg:2_stage_training},
    and achieves a loss of $d^{-c}/C$ for some constant $c \geq 1$ and some sufficiently large constant $C > 0$ at the end of the second stage. 
    Suppose further, that we train a student model $f_s$ of size $m_s = \tilde{\Theta}(2^k k)$ using the following two strategies:
\begin{enumerate}
    \item \textbf{Random-projection curriculum extraction:} 
    Train the first layer of the student with a random projection of the first layer of the teacher to the right output dimension, and then train the entire student network with the final teacher network.
    \item \textbf{One-shot Distillation:} Train with the teacher network throughout.
\end{enumerate}

Then,
\begin{enumerate}
    \item 
    \label{item:sufficient_our_distillation_scheme}
    Under our distillation scheme, the total sample complexity to reach a loss of $\epsilon$ with probability $99\%$ is 
    $\Theta(2^{O(k)} \poly(d,k)\epsilon^{-2} \log(k/\epsilon)).$
    \item The necessary sample complexity under distillation is at least
    \label{item:necessary_one_shot_distill}
    $\Omega\left(d^{\min(2c, k-1)}\right)$.
    
\end{enumerate}
\end{theorem}
 The key difference between the two is that, in one-shot distillation, the student must identify one of \(\Omega(d^k)\) possible parity functions from scratch. In contrast, our scheme splits the learning into two phases: identifying the support and learning the final function. Initially, the gradients of the distillation loss guide the student in detecting the support of the sparse parity via the bottom layer. With the support identified, the student only needs to select from \(O(2^k)\) possible parities.

\subsubsection{Proof Overview}
\label{subsec:proof_overview}

We begin by assuming, without loss of generality, that the support of the unknown parity is $S=[k]$. 
Our goal is to prove two main results: a lower bound for one-shot distillation and an upper bound on the sample complexity of our curriculum extraction scheme. 
The lower bound follows directly from \citet{panigrahi2024progressive} (Theorem B.1); here we focus on a high level sketch of the upper bound.


To achieve this, our approach explicitly separates two tasks: first, \emph{support recovery} (identifying the relevant coordinates for the 
$k$-sparse parity), and second, \emph{loss minimization} (optimizing the network to approximate the parity function). This separation is central to reducing the sample complexity. 

Our key insight is to exploit a property of the teacher's first-layer weights. When the teacher is trained using the two-stage process, the first-layer weights have significantly larger magnitudes on the true support than on the out-of-support coordinates (see \Cref{lem:effect_of_1_step_on_w}). This gap is crucial for ensuring that the second stage of training can find a network with low loss.
In particular, Theorem 4 from \citet{barak2022hidden} (restated in \Cref{lem:second_stage_conclusion}) shows that the \(k\)-sparse parity can be accurately approximated by training only the top layer of the student, provided that the hidden layer satisfies some mild conditions, such as being large enough (at least \(\tilde{\Omega}(2^k k)\)) and satisfying a clear gap between the in-support and out-of-support coordinates of the weights at the end of the first stage of training.

In our approach, in the first training step, we transfer the gap from the teacher's first-layer weights to the student's first-layer weights using the gradients of the distillation loss. This transfer enables support recovery using far fewer samples than would be required if we trained the entire network from scratch. Once the support is recovered, the second stage of training, where we update only the top layer weights, proceeds exactly as in \cite{panigrahi2024progressive, barak2022hidden}.

By an appropriate choice of the regularization parameter $\lambda$, we can ensure that the student's first layer weights $(f_s^{(1)})$ become proportional to the gradients after the first update. This means that it is sufficient to demonstrate a clear gap between the magnitudes of the in-support and out-of-support coordinates of the gradient at initialization. In fact, \Cref{lem:teach_gap_implies_student_gap} shows that the gradient of the distillation loss with respect to \(\bw_i\) is proportional to
$\E_{\bx}\Bigl[(Af_t^{(1)}(\bx))_i \, \bx + (Af_t^{(1)}(\bx))_i \, \maj(\bw_i \odot \bx)\, \bx\Bigr].$
Thus, if we can show that the magnitude of the $j$-th coordinate for $j \in S$ (in-support) is significantly larger than for $j \notin S$ (out-of-support) the student can effectively recover the support from the gradient information.

For the remainder of this discussion, we focus on showing this gap for the first term, since it will turn out that the second term is dominated by this first term.
The \(j\)-th coordinate of the first term is given by $\E_{\bx}\Bigl[(Af_t^{(1)}(\bx))_i\, x_j\Bigr].$
Due to the symmetric initialization (defined in \Cref{def:sym_initialization}),  the weights after the first training stage retain this symmetry (see, \Cref{lem:effect_of_1_step_on_w}).
As a result, we can rewrite the first term as a sum of scaled Rademacher random variables \(\{A_{i\ell}\}_{\ell=1}^{m_t/2}\), i.e., 
$\E_{\bx}\Bigl[(Af_t^{(1)}(\bx))_i\, x_j\Bigr] 
= \sum_{\ell=1}^{m_t/2} A_{i\ell}~\E_{\bx}\Bigl[\left( \sigma(\bx\cdot \bw_\ell + b_\ell) - \sigma(-\bx\cdot \bw_\ell + b_\ell) \right)~x_j\Bigr]$. Let $s_{\ell j}$ be scaling factors defined as, $s_{\ell j} := \E_{\bx}\Bigl[\left( \sigma(\bx\cdot \bw_\ell + b_\ell) - \sigma(-\bx\cdot \bw_\ell + b_\ell) \right)~x_j\Bigr]$. 

Then, $ \E_{\bx}\Bigl[(Af_t^{(1)}(\bx))_i\, x_j\Bigr] = \sum_{\ell=1}^{m_t/2} A_{i\ell} s_{\ell j}$. By the central limit theorem, the sum of scaled Rademacher random variables behaves like a mean-zero Gaussian with variance $\sigma_j^2 := \sum_{\ell=1}^{m_t/2} s_{\ell j}^2/m_t^2$.
\Cref{lem:bounds_on_coeffs} then guarentees that the scaling factors are significantly larger when \(j\) is in the true support \(S\) than when \(j\) is not. 
This, in turn, 
implies that the standard deviations of the in-support $j$ are much larger than those for the out-of-support $j$, more precisely, $\sigma_j > \frac{1}{k\sqrt{m_t}}$ for $j$ in the support, and $\sigma_j < \frac{1}{kd\sqrt{m_t}}$ for $j$ out of the support (see \Cref{fig:correlation_plot} for the corresponding distributions in our trained network).


Next, by applying anticoncentration and concentration inequalities for sums of Rademacher random variables
- which behave similar to those of a Gaussian - to the sum
$\sum_{\ell=1}^{m_t/2} A_{i\ell} s_{\ell j}$
we show that, with high probability over the randomness of the projection $A$, this variance gap carries over to a proportional gap in a sufficiently large subset of coordinates of $\E_{\bx}\Bigl[(Af_t^{(1)}(\bx))_i\,\bx\Bigr]$ 
. This result is summarized informally below; the formal statement of which is provided in \Cref{lem:degree_1_gap}.

\begin{lemma}[Correlation Gap (Informal)]
\label{lem:corr_gap_informal}
Let $d \geq \Omega(k^4)$ and $m_t \geq \Omega(k^2)$. 
Then, with probability 
$99\%$ over the randomness of initialization, an independently drawn subset of at least $\tilde \Omega(2^k k)$ coordinates 
of the projected teacher network satisfy $\left|\E_\bx[(Af_t)_i (\bx) x_j]\right| > \Omega((k^2 \sqrt{m_t})^{-1})$ for all $j$ in the support of the unknown sparse parity, and $\left|\E_\bx[(Af_t)_i (\bx) x_j]\right| \leq O((k^3 \sqrt{m_t})^{-1})$ for $j$ that are out-of-support.
\end{lemma}

Above, we outlined the argument to show a gap for the first term of the gradient. \Cref{lem:expansion_of_f_r} demonstrates that the second term of the gradient is dominated by the first, 
and so the gap witnessed by the first term transfers to the gradient. 
This gap drives our support recovery.

Now, we must verify that this gap persists when using empirical estimates. To do so, it is sufficient to ensure that the scaling factors of the variables $\{ A_{i\ell} \}_{\ell=1}^{m_t/2}$, are close to the true expectations.
If this closeness holds, the empirical standard deviations will inherit the gap observed by the $\sigma_j$, and the concentration and anticoncentration properties will continue to apply -- thus preserving the conclusion of \Cref{lem:corr_gap_informal}.

In \Cref{lem:emp_grad_gap}, we use Hoeffding's inequality (\Cref{lem:hoeffding}) to show that the empirical scaling factors converge to their true values with only $O((kd)^2 \log(m_td))$ samples. Once the gradients are updated using these empirical estimates, the student's weights exhibit a gap -- after rescaling by $\sqrt{m_t}$ this gap is $\Omega(1/k^2)$.
This result allows us to invoke the analysis from \cite{barak2022hidden, panigrahi2024progressive} for the second stage, which requires an additional $\tilde O(2^{O(k)}\poly(d)/\eps^2)$ samples to achieve a loss of $\eps$. Thus, we learn the parity with a significantly smaller overall sample complexity compared to one-shot distillation.

At this point, we note an important difference between our proof and the one for progressive distillation in \citet{panigrahi2024progressive}. In their work, the weights after the first stage of training adjust to the initialized top layer. This dependence allows them to more easily demonstrate a gap in the gradients, without having to deal with the randomness of a projection. In fact, we also observe the effect of being able to tune to the current top layer in \Cref{fig:student_accuracies} (a), where we see that progressive distillation is able to make some progress even during the stage where we only tune the bottom layer weights. 
In our setting, in the first phase of training, we train the bottom layer of the student independently of the top layer, requiring us to rely on a different argument.

\subsection{Experiments} 
\label{subsec:sparse_parity_experiments}
We investigate curriculum extraction for the problem of learning sparse parities for a Multi-Layer Perceptron (MLP) and Transformer architectures, which we describe below:

\paragraph{Student and Teacher Architectures}
For the \emph{Multi-Layer Perceptron,} both the student and teacher are two-layer MLPs with the teacher network having a hidden dimension of $5 \times 10^4$ and the student having hidden dimension $100$.  For the \emph{transformer architecture}, the transformer configuration has matching embedding dimensions ($256$ dimensions for both teacher and student); however, the teacher has 32 attention heads and the student has only 4. \new{Both student as well as teacher architectures use two decoder blocks followed by a linear projection layer.}
\paragraph{Training and Evaluation}
Our distillation loss is the Mean Squared Error (MSE) and the final checkpoint training is done using the Cross-Entropy loss. 
We measure performance of our model by looking at the accuracy. 
\subsubsection{Discussion}
\paragraph{Multi-Layer Perceptron:}
In \Cref{fig:student_accuracies} (a), we compare the performance of MLPs trained using one-shot distillation, layer-wise curriculum extraction, and the progressive distillation approach from \citet{panigrahi2024progressive}. In both curriculum extraction and progressive distillation, we switch to the final checkpoint at iteration \( 5 \times 10^5 \).

We observe that after \(10^6\) iterations, both methods perform comparably, with progressive distillation achieving slightly higher accuracy. Notably, progressive distillation shows early improvement before the \(5 \times 10^5\) iteration mark, likely due to the MLP's bottom layer quickly tuning to the top layer. In the second phase of training, the top layer benefits from a well-optimized starting point. 
With curriculum extraction on the other hand, the bottom layer starts randomly initialized and uncorrelated with the top layer, limiting early gains. However, once we begin training the top layer, performance improves rapidly, matching progressive distillation. In contrast, one-shot distillation shows no significant improvement, even after \(2 \times 10^6\) iterations.
In fact, we observe in \Cref{fig:mlp_proj_vs_layer_correlation} that 
our curriculum extraction method leads to the more information about the support being transferred to the underlying network than progressive distillation.


\paragraph{Transformer:}
In \Cref{fig:student_accuracies} (b), we compare the performance of the transformer model trained using one-shot distillation and our curriculum extraction scheme, with a checkpoint at $10^5$ iterations. We see that this leads to significantly improved student performance over the one-shot distillation case. 

\section{BERT Language Modeling}
\label{sec:bert}
\begin{figure*}[t]
    \centering
    \begin{minipage}[t]{0.31\linewidth}
        \centering
        \hspace*{-2mm}
        \includegraphics[width=1.1\textwidth]{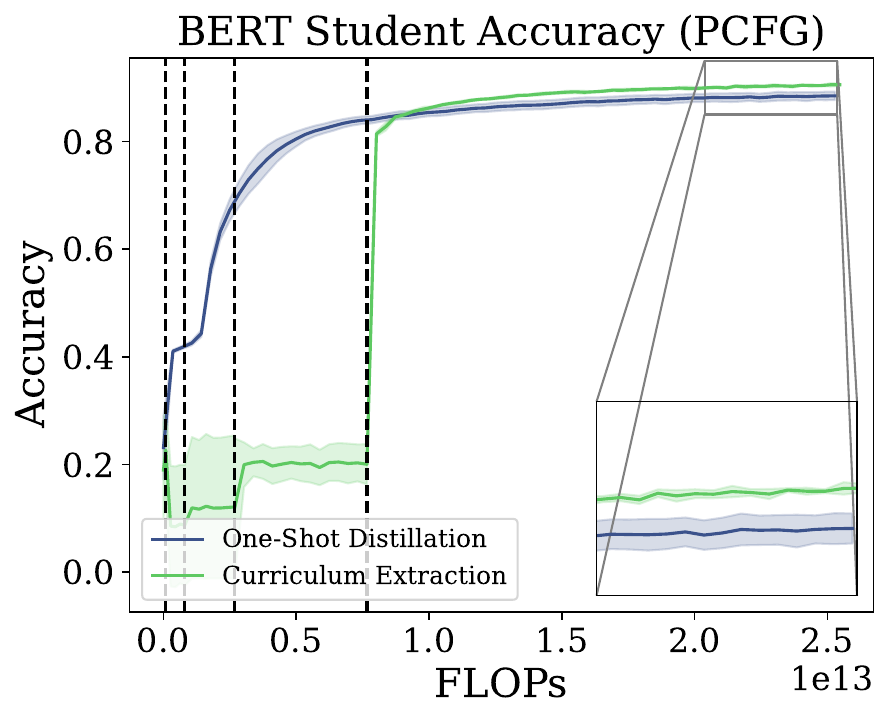}
        \vspace{0.3em} 
        \centering\scriptsize\textbf{(a)} 
        \label{fig:bert_pcfg_flops_vs_acc}
    \end{minipage}
        \hfill
        \begin{minipage}[t]{0.31\linewidth}
    \centering
    \hspace*{-2mm}
    \includegraphics[width=1.1\textwidth]{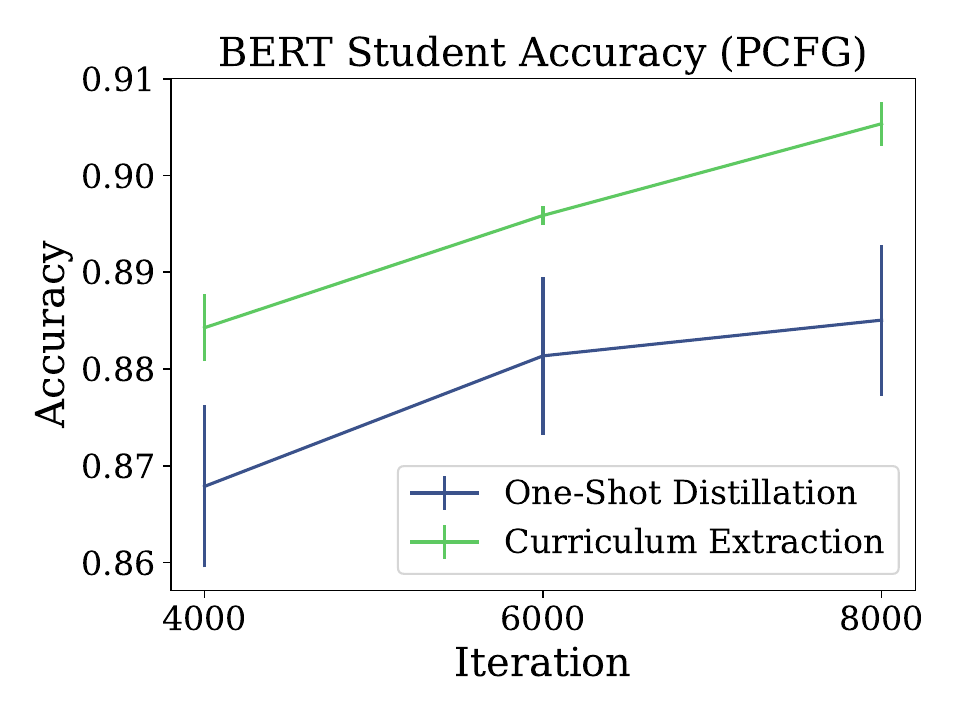}
        \vspace{0.3em} 
        \centering\scriptsize\textbf{(b)} 
    \label{fig:many_vs_one_bert_pcfg}
    \end{minipage}
    \hfill
    \begin{minipage}[t]{0.31\linewidth}
    \centering
    \hspace*{-3mm}
    \includegraphics[width=1.1\textwidth]{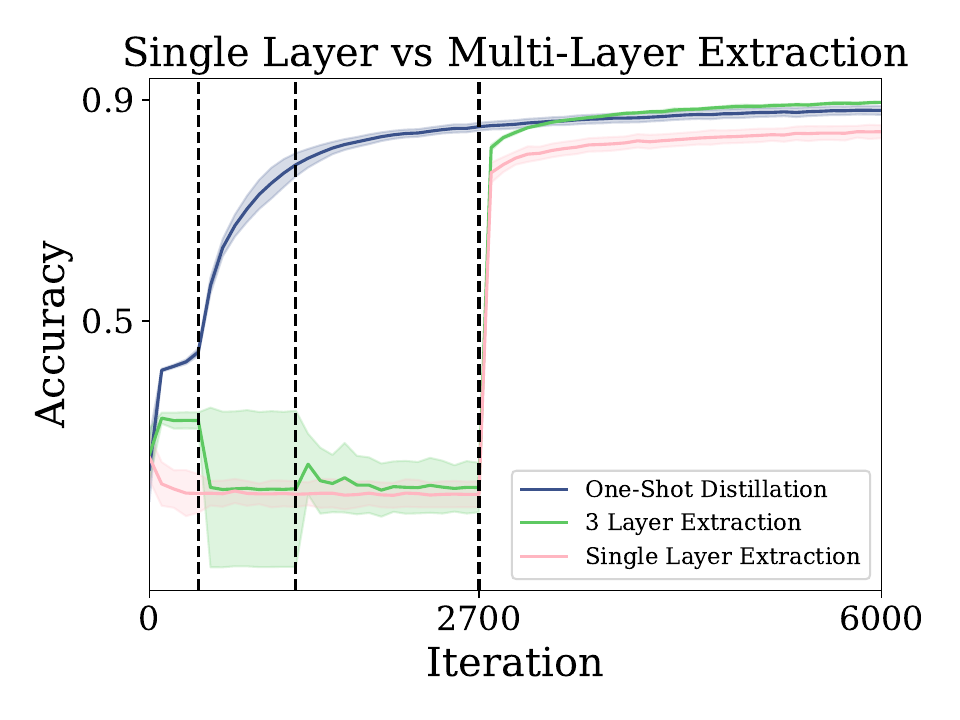}
        \vspace{0.3em} 
        \centering\scriptsize\textbf{(c)} 
    \label{fig:bert_pcfg_layerwise_vs_oneshot_acc}
    \end{minipage}

    \caption{%
    \textbf{PCFG Experiments on BERT.}
    The dashed vertical lines indicate iterations where the layer being distilled from is changed.
    (a) Soon after the final checkpoint, curriculum extraction achieves a larger accuracy in the same number of FLOPs, when compared to one-shot distillation.
    (b) We compare curriculum extraction to one-shot distillation across three models trained with two, three, and four-stage curricula at 4000, 6000, and 8000 iterations, respectively. Curriculum extraction consistently outperforms one-shot distillation at all scales. 
    (c) We compare curriculum extraction performance by varying the number of layers. With a fixed budget of 6000 iterations (2700 for extraction, 3300 for full network training), extracting from three layers outperforms one-shot distillation, and using a single layer.
}
    \label{fig:pcfg_bert}
\end{figure*}
We use BERT models trained for masked prediction, a task that involves predicting the tokens hidden (masked) in an input sequence. Similar to \citet{panigrahi2024progressive} we study this task for learning
probabilistic context-free grammars (PCFGs) and natural language modeling on the Wikipedia dataset. We define the masked token prediction task below.

\begin{problem}[Masked Token Prediction]
    Let $v$ be a vocabulary (including the token \([mask]\)), and let $x$ be a sequence of length $h$. We randomly choose a fraction $(30\%)$ of the positions $M\subseteq [h]$ to mask, with each position included independently with probability $p$. We create a masked input $x \setminus M$ by replacing tokens in $M$ with \([mask]\), a random token, or leaving them unchanged with respective probabilities $80\%,\,10\%,\,10\%$.
The model is then trained via a cross-entropy objective to predict the original tokens at those masked positions.
\end{problem}
For sequence-to-sequence modelling our teacher and student networks map from sequences to real vectors $f_t, f_s \colon v^h \;\to\; \mathbb{R}^{h \times C}$. The teacher’s output distribution at position $i$ is given by $p_T^{(i)}(\bx; \tau) \;=\; \mathrm{softmax}\bigl([f_t(\bx)]_i \,/\, \tau\bigr),$ and $p_S^{(i)}$ is defined analogously for the student.
We set the temperature $\tau$ to $10^{-4}$ for experiments on sparse parity and PCFG, and to $10^{-20}$ for Wikipedia.

\paragraph{Training and Evaluation}
For the masked prediction task, we use $ \ell(\bx; f_S) = \mathbb{E}_{M} \Bigl[\frac{1}{|M|} \sum_{i \in M} \mathrm{KL}\bigl(e_{x_i} \,\|\, p_S^{(i)}(\bx \setminus M; \tau)\bigr)\Bigr],$ where $e_{y}$ is a one-hot vector with a 1 at index $y$ for the final layer training; and the MSE for the indetermediate layers as our distillation loss. Our final performance is measured using the top-1 accuracy on the masked tokens.

\paragraph{Student and Teacher Architectures}
For the \emph{PCFG tasks}, the teacher model uses a BERT-style architecture 
with a 256-dimensional embedding, 32 attention heads, and 4 transformer blocks, trained with a batch size of 512. The student model retains the teacher’s 4-block architecture and 512 batch size but reduces the embedding dimension to 64 and the number of attention heads to 8.\\
For \emph{Wikipedia} language modeling, the teacher model follows a standard BERT-large configuration, with 768-dimensional embeddings, 12 attention heads, and 12 transformer blocks, trained with a batch size of 256. The student model preserves the teacher’s 12-block depth and 256 batch size but reduces the embedding dimension to 256  and the number of attention heads to 4.

\paragraph{Probabilistic Context-Free Grammar}
A Probabilistic Context-Free Grammar (PCFG) generates sentences using a hierarchical tree structure, defined by non-terminal symbols, rules, a probability distribution over the rules, and a vocabulary of terminal symbols. PCFGs have been used as mechanistic proxies for language data \citep{panigrahi2024progressive, allen2023physics, zhao2023transformers}. We focus on masked token prediction for synthetic data from the cfg3b PCFG \citep{allen2023physics} (defined in \Cref{app:pcfg}). Our layer-wise curriculum schedule outperforms one-shot distillation. To confirm this improvement is due to the curriculum and not increased training bandwidth, we compare models trained with different curricula but the same time steps.




\subsection{Discussion}
\paragraph{PCFG Experiments:}

In \Cref{fig:pcfg_bert} (a), curriculum extraction shows FLOPS savings compared to one-shot distillation, outperforming it after \(\approx 0.8 \times 10^{13}\) FLOPS. Additionally, its accuracy improves at every checkpoint.

To establish that the reason for our improved performance is our scheme, 
In \Cref{fig:pcfg_bert} (b), we implement two distinct curriculum strategies for BERT distillation to train the student network while maintaining equivalent \emph{bandwidth} across our experiments.
We compare single-layer and multi-layer extraction, both with 6000 training steps and 2700 dedicated to curriculum extraction. The single-layer model uses one checkpoint at the end, while the three-layer model has checkpoints at 400, 1200, and 2700 steps. Extracting from three layers improves performance, while single-layer extraction appears to perform worse than one-shot distillation, possibly due to overfitting between the student's bottom layer and the teacher's upper layers.

In \Cref{fig:pcfg_bert} (c) we see similar performance gains to those in \Cref{fig:pcfg_bert} (b) for higher bandwidth experiments -- for two, three and four-stage curricula. 
The two-stage curriculum skips the linear projection and first transformer layer for iterations 1 through 500, skips just the top linear layer for iterations \new{501 to 1500, and finally trains the entire network for iterations 1501 through 4000}.  
\new{The three-stage curriculum skips two encoder blocks for iterations 1 through 400, skips one encoder block for iterations 401 through 1200, and skips the final projection layer for iterations 1201 through 2700, finally training the entire network for iterations 2701 to 6000.}
The four-stage curriculum skips four encoder blocks for iterations 1 through 200, then three blocks for iterations 201 through 700, then two blocks for iterations 701 through 1500, one block for iteration 1501 through 3000 and finally just the final linear layer until iteration 8000.
    
    
\paragraph{Wikipedia}
    In \Cref{fig:student_accuracies} (c) we see that our extraction scheme also works extremely well on real-world data.
    We train our BERT model for 500 iterations while skipping 4 encoder blocks, a further 1000 iterations while skipping 3 encoder blocks, a further 2000 iterations skipping one encoder block, 4000 iterations skipping the final projection layer and finally the full network for 16500 iterations.

	\bibliographystyle{alpha}
	\bibliography{main}

\newcommand{\etalchar}[1]{$^{#1}$}
\begin{thebibliography}{DAGY{\etalchar{+}}25}

\bibitem[AAB{\etalchar{+}}24]{abdin2024phi4technicalreport}
Marah Abdin, Jyoti Aneja, Harkirat Behl, S{\'e}bastien Bubeck, Ronen Eldan, Suriya Gunasekar, Michael Harrison, Russell~J Hewett, Mojan Javaheripi, Piero Kauffmann, et~al.
\newblock Phi-4 technical report.
\newblock {\em arXiv preprint arXiv:2412.08905}, 2024.

\bibitem[ALZ{\etalchar{+}}20]{aguilar2020knowledge}
Gustavo Aguilar, Yuan Ling, Yu~Zhang, Benjamin Yao, Xing Fan, and Chenlei Guo.
\newblock Knowledge distillation from internal representations.
\newblock In {\em Proceedings of the AAAI conference on artificial intelligence}, volume~34, pages 7350--7357, 2020.

\bibitem[APP{\etalchar{+}}18]{anil2020largescaledistributedneural}
Rohan Anil, Gabriel Pereyra, Alexandre Passos, Robert Ormandi, George~E Dahl, and Geoffrey~E Hinton.
\newblock Large scale distributed neural network training through online distillation.
\newblock {\em arXiv preprint arXiv:1804.03235}, 2018.

\bibitem[AZL23]{allen2023physics}
Zeyuan Allen-Zhu and Yuanzhi Li.
\newblock Physics of language models: Part 1, context-free grammar.
\newblock {\em arXiv preprint arXiv:2305.13673}, 2023.

\bibitem[BEG{\etalchar{+}}22]{barak2022hidden}
Boaz Barak, Benjamin~L. Edelman, Surbhi Goel, Sham~M. Kakade, eran malach, and Cyril Zhang.
\newblock Hidden progress in deep learning: {SGD} learns parities near the computational limit.
\newblock In Alice~H. Oh, Alekh Agarwal, Danielle Belgrave, and Kyunghyun Cho, editors, {\em Advances in Neural Information Processing Systems}, 2022.

\bibitem[BLCW09]{bengio2009curriculum}
Yoshua Bengio, J{\'{e}}r{\^{o}}me Louradour, Ronan Collobert, and Jason Weston.
\newblock Curriculum learning.
\newblock In {\em ICML}, pages 41--48, 2009.

\bibitem[CH19]{cho}
Jang~Hyun Cho and Bharath Hariharan.
\newblock On the efficacy of knowledge distillation.
\newblock In {\em 2019 IEEE/CVF International Conference on Computer Vision (ICCV)}, pages 4793--4801, 2019.

\bibitem[DAGY{\etalchar{+}}25]{deepseekai2025deepseekr1incentivizingreasoningcapability}
DeepSeek-AI, Daya Guo, Dejian Yang, Haowei Zhang, Junxiao Song, Ruoyu Zhang, Runxin Xu, Qihao Zhu, Shirong Ma, Peiyi Wang, Xiao Bi, Xiaokang Zhang, Xingkai Yu, Yu~Wu, Z.~F. Wu, Zhibin Gou, Zhihong Shao, Zhuoshu Li, Ziyi Gao, Aixin Liu, Bing Xue, Bingxuan Wang, Bochao Wu, Bei Feng, Chengda Lu, Chenggang Zhao, Chengqi Deng, Chenyu Zhang, Chong Ruan, Damai Dai, Deli Chen, Dongjie Ji, Erhang Li, Fangyun Lin, Fucong Dai, Fuli Luo, Guangbo Hao, Guanting Chen, Guowei Li, H.~Zhang, Han Bao, Hanwei Xu, Haocheng Wang, Honghui Ding, Huajian Xin, Huazuo Gao, Hui Qu, Hui Li, Jianzhong Guo, Jiashi Li, Jiawei Wang, Jingchang Chen, Jingyang Yuan, Junjie Qiu, Junlong Li, J.~L. Cai, Jiaqi Ni, Jian Liang, Jin Chen, Kai Dong, Kai Hu, Kaige Gao, Kang Guan, Kexin Huang, Kuai Yu, Lean Wang, Lecong Zhang, Liang Zhao, Litong Wang, Liyue Zhang, Lei Xu, Leyi Xia, Mingchuan Zhang, Minghua Zhang, Minghui Tang, Meng Li, Miaojun Wang, Mingming Li, Ning Tian, Panpan Huang, Peng Zhang, Qiancheng Wang, Qinyu Chen, Qiushi Du, Ruiqi Ge, Ruisong
  Zhang, Ruizhe Pan, Runji Wang, R.~J. Chen, R.~L. Jin, Ruyi Chen, Shanghao Lu, Shangyan Zhou, Shanhuang Chen, Shengfeng Ye, Shiyu Wang, Shuiping Yu, Shunfeng Zhou, Shuting Pan, S.~S. Li, Shuang Zhou, Shaoqing Wu, Shengfeng Ye, Tao Yun, Tian Pei, Tianyu Sun, T.~Wang, Wangding Zeng, Wanjia Zhao, Wen Liu, Wenfeng Liang, Wenjun Gao, Wenqin Yu, Wentao Zhang, W.~L. Xiao, Wei An, Xiaodong Liu, Xiaohan Wang, Xiaokang Chen, Xiaotao Nie, Xin Cheng, Xin Liu, Xin Xie, Xingchao Liu, Xinyu Yang, Xinyuan Li, Xuecheng Su, Xuheng Lin, X.~Q. Li, Xiangyue Jin, Xiaojin Shen, Xiaosha Chen, Xiaowen Sun, Xiaoxiang Wang, Xinnan Song, Xinyi Zhou, Xianzu Wang, Xinxia Shan, Y.~K. Li, Y.~Q. Wang, Y.~X. Wei, Yang Zhang, Yanhong Xu, Yao Li, Yao Zhao, Yaofeng Sun, Yaohui Wang, Yi~Yu, Yichao Zhang, Yifan Shi, Yiliang Xiong, Ying He, Yishi Piao, Yisong Wang, Yixuan Tan, Yiyang Ma, Yiyuan Liu, Yongqiang Guo, Yuan Ou, Yuduan Wang, Yue Gong, Yuheng Zou, Yujia He, Yunfan Xiong, Yuxiang Luo, Yuxiang You, Yuxuan Liu, Yuyang Zhou, Y.~X. Zhu,
  Yanhong Xu, Yanping Huang, Yaohui Li, Yi~Zheng, Yuchen Zhu, Yunxian Ma, Ying Tang, Yukun Zha, Yuting Yan, Z.~Z. Ren, Zehui Ren, Zhangli Sha, Zhe Fu, Zhean Xu, Zhenda Xie, Zhengyan Zhang, Zhewen Hao, Zhicheng Ma, Zhigang Yan, Zhiyu Wu, Zihui Gu, Zijia Zhu, Zijun Liu, Zilin Li, Ziwei Xie, Ziyang Song, Zizheng Pan, Zhen Huang, Zhipeng Xu, Zhongyu Zhang, and Zhen Zhang.
\newblock Deepseek-r1: Incentivizing reasoning capability in llms via reinforcement learning, 2025.

\bibitem[HM19]{hewitt2019structural}
John Hewitt and Christopher~D Manning.
\newblock A structural probe for finding syntax in word representations.
\newblock In {\em Proceedings of the 2019 Conference of the North American Chapter of the Association for Computational Linguistics: Human Language Technologies, Volume 1 (Long and Short Papers)}, pages 4129--4138, 2019.

\bibitem[HRM{\etalchar{+}}23]{supervision_complexity}
Hrayr Harutyunyan, Ankit Rawat, Aditya Menon, Kim Seungyeon, and Sanjiv Kumar.
\newblock Supervision complexity and its role in knowledge distillation, 01 2023.

\bibitem[HVD15]{Hinton2015DistillingTK}
Geoffrey~E. Hinton, Oriol Vinyals, and Jeffrey Dean.
\newblock Distilling the knowledge in a neural network.
\newblock {\em ArXiv}, abs/1503.02531, 2015.

\bibitem[JSM{\etalchar{+}}23]{jiang2023mistral7b}
Albert~Q Jiang, Alexandre Sablayrolles, Arthur Mensch, Chris Bamford, Devendra~Singh Chaplot, Diego de~las Casas, Florian Bressand, Gianna Lengyel, Guillaume Lample, Lucile Saulnier, et~al.
\newblock Mistral 7b.
\newblock {\em arXiv preprint arXiv:2310.06825}, 2023.

\bibitem[JYS{\etalchar{+}}20]{jiao2020tinybert}
Xiaoqi Jiao, Yichun Yin, Lifeng Shang, Xin Jiang, Xiao Chen, Linlin Li, Fang Wang, and Qun Liu.
\newblock Tinybert: Distilling {BERT} for natural language understanding.
\newblock {\em EMNLP Findings}, pages 4163--4174, 2020.

\bibitem[KB17]{kocmi2017curriculum}
Tom Kocmi and Ondrej Bojar.
\newblock Curriculum learning and minibatch bucketing in neural machine translation.
\newblock 2017.

\bibitem[KPK10]{kumar2010self}
M~Kumar, Benjamin Packer, and Daphne Koller.
\newblock Self-paced learning for latent variable models.
\newblock {\em Advances in neural information processing systems}, 23, 2010.

\bibitem[LZZ{\etalchar{+}}23]{liang2023less}
Chen Liang, Simiao Zuo, Qingru Zhang, Pengcheng He, Weizhu Chen, and Tuo Zhao.
\newblock Less is more: Task-aware layer-wise distillation for language model compression.
\newblock In {\em International Conference on Machine Learning}, pages 20852--20867. PMLR, 2023.

\bibitem[{Met}24]{meta2024llama}
{Meta AI}.
\newblock Llama 3.
\newblock 2024.
\newblock \url{https://ai.meta.com/blog/meta-llama-3/}.

\bibitem[MFL{\etalchar{+}}20]{mirzadeh2019improvedknowledgedistillationteacher}
Seyed~Iman Mirzadeh, Mehrdad Farajtabar, Ang Li, Nir Levine, Akihiro Matsukawa, and Hassan Ghasemzadeh.
\newblock Improved knowledge distillation via teacher assistant.
\newblock In {\em Proceedings of the AAAI conference on artificial intelligence}, volume~34, pages 5191--5198, 2020.

\bibitem[Ope24]{openai-o1mini}
OpenAI.
\newblock Openai o1-mini, 2024.

\bibitem[PLM{\etalchar{+}}24]{panigrahi2024progressive}
Abhishek Panigrahi, Bingbin Liu, Sadhika Malladi, Andrej Risteski, and Surbhi Goel.
\newblock Progressive distillation induces an implicit curriculum.
\newblock {\em arXiv preprint arXiv:2410.05464}, 2024.

\bibitem[PSL{\etalchar{+}}24]{panigrahi2024efficient}
Abhishek Panigrahi, Nikunj Saunshi, Kaifeng Lyu, Sobhan Miryoosefi, Sashank Reddi, Satyen Kale, and Sanjiv Kumar.
\newblock Efficient stagewise pretraining via progressive subnetworks.
\newblock {\em arXiv preprint arXiv:2402.05913}, 2024.

\bibitem[SCGL19]{Sun2019PatientKD}
S.~Sun, Yu~Cheng, Zhe Gan, and Jingjing Liu.
\newblock Patient knowledge distillation for bert model compression.
\newblock In {\em Conference on Empirical Methods in Natural Language Processing}, 2019.

\bibitem[SZT17]{shwartz2017opening}
Ravid Shwartz-Ziv and Naftali Tishby.
\newblock Opening the black box of deep neural networks via information.
\newblock {\em arXiv preprint arXiv:1703.00810}, 2017.

\bibitem[Tea24]{geminiteam2024gemini15unlockingmultimodal}
Gemini Team.
\newblock Gemini 1.5: Unlocking multimodal understanding across millions of tokens of context, 2024.

\bibitem[Ten19]{tenney2019bert}
I~Tenney.
\newblock Bert rediscovers the classical nlp pipeline.
\newblock {\em arXiv preprint arXiv:1905.05950}, 2019.

\bibitem[VTM{\etalchar{+}}19]{voita2019analyzing}
Elena Voita, David Talbot, Fedor Moiseev, Rico Sennrich, and Ivan Titov.
\newblock Analyzing multi-head self-attention: Specialized heads do the heavy lifting, the rest can be pruned.
\newblock {\em arXiv preprint arXiv:1905.09418}, 2019.

\bibitem[ZPGA23]{zhao2023transformers}
Haoyu Zhao, Abhishek Panigrahi, Rong Ge, and Sanjeev Arora.
\newblock Do transformers parse while predicting the masked word?
\newblock {\em arXiv preprint arXiv:2303.08117}, 2023.

\end{thebibliography}
	
	\newpage
	\appendix
	\section{Organization}
 In \Cref{app:prelims} we recall  essential definitions and probabilistic tools that we will use throughout our analysis. 
 In \Cref{app:teacher_training} we revisit the analysis of the teacher model as presented in \cite{panigrahi2024progressive,barak2022hidden}. 
 In \Cref{app:student_training}, we analyze the sample complexity required for our student model to perform well under our proposed framework. 
 Finally, \Cref{app:pcfg} formally defines a probabilistic context-free grammar (PCFG) and reviews the definition of the \textbf{cfg3b} grammar from \cite{allen2023physics}, which serves as the foundation for our experimental evaluation.

\section{Preliminaries}
\label{app:prelims}
This section lays the mathematical groundwork for our analysis. We begin by defining the Fourier expansion of Boolean functions. Next, we introduce important properties of the ReLU function and present key probabilistic tools such as the Berry–Esseen theorem and Hoeffding’s inequality. These results will be repeatedly used in the subsequent analysis to control gradient estimation and sample complexity.

We define the Fourier expansion of a boolean function below. 
\begin{definition}[Fourier Expansion of a Boolean Function]
\label{def:fourier_expansion}
    Let \(f: \{-1, 1\}^n \to \mathbb{R}\) be a Boolean function. The \textbf{Fourier coefficients} of \(f\) are defined as:
\[
\hat{f}(S) = \mathbb{E}_{\bx \sim \{-1, 1\}^n} \left[ f(\bx) \cdot \chi_S(\bx) \right],
\]
where:
\begin{itemize}
    \item \(S \subseteq [n]\) is a subset of the input coordinates,
    \item \(\chi_S(\bx) = \prod_{i \in S} x_i\) is the \textbf{parity function} 
    corresponding to \(S\),
\end{itemize}

The function \(f\) can then be expressed in terms of its Fourier expansion:
\[
f(\bx) = \sum_{S \subseteq [n]} \hat{f}(S) \cdot \chi_S(\bx).
\]
\end{definition}


\subsection{Properties of the ReLU Function}

We will need the following properties of the ReLU function. 

\begin{lemma}[Properties of $\phi_b(t)$]
\label{lem:prop_phi}
Let $\sigma(t) := \max(0, t)$, and let $a, b \in \R$.
Then $\phi_b(a) := \sigma(a + b) - \sigma(-a+b)$ satisfies the following:
\begin{enumerate}
    \item $\phi_b(0) = 0$
    \item $\phi_b(-t) = -\phi_b(t)$
    \item $\phi_b(t)$ is monotonically non-decreasing in $t$. 
\end{enumerate}
\end{lemma}
\begin{proof}
The proofs follow by the definition of $\phi_b(a)$. 
\begin{enumerate}
    \item $\phi_b(0) = \sigma(b) - \sigma(b)$.
    \item $\phi_b(-t) = \sigma(-t+b) - \sigma(t+b) = -\phi_b(t)$.
    \item This follows from the fact that $\sigma(t)$ is monotonically non-decreasing. If $t_1 \leq t_2$, observe that
    \begin{align*}
        \phi_b(t_1) 
        &= \sigma(t_1 + b) - \sigma(-t_1+b)\\
        &\leq \sigma(t_2 + b) - \sigma(-t_2 + b)\\
        &= \phi_b(t_2).
    \end{align*}
\end{enumerate}
\end{proof}





\subsection{Probability Facts}
We will need the following anticoncentration and concentration inequalities for sums of scaled Rademacher random variables. 

\begin{theorem}[Berry--Esseen]
\label{thm:berry-esseen}
Let $X_1, X_2, \dots, X_n$ be independent random variables satisfying
\[
\mathbb{E}[X_i] = 0,\quad \mathbb{E}[X_i^2] = \sigma_i^2 > 0,\quad \text{and} \quad \mathbb{E}[|X_i|^3] < \infty, \quad \text{for } i=1,\dots,n.
\]
Define
\[
S_n = \sum_{i=1}^n X_i, \quad \sigma^2 = \sum_{i=1}^n \sigma_i^2, \quad \text{and} \quad \rho_n = \sum_{i=1}^n \mathbb{E}[|X_i|^3].
\]
Then for all $x\in\mathbb{R}$, we have
\[
\left| \Pr\left[ \frac{S_n}{\sigma} \le x \right] - \Phi(x) \right| \le C \frac{\rho_n}{\sigma^3},
\]
where $\Phi(x)$ is the standard normal cumulative distribution function and $C>0$ is an absolute constant.
\end{theorem}

\begin{theorem}[Anticoncentration for Rademacher Sums]
\label{thm:anticoncentration-rademacher}
Let \(x_1,\dots,x_n\) be independent Rademacher random variables 
and let \(c_1,\dots,c_n\) be real numbers. Define
\[
X = \sum_{i=1}^n c_i x_i,\quad \sigma^2 = \sum_{i=1}^n c_i^2.
\]
Then, for any \(t\in\mathbb{R}\) and any \(\delta>0\), we have
\[
\Pr\Bigl[X\in [t,t+\delta\sigma]\Bigr] \le \frac{\delta}{\sqrt{2\pi}} + 2C\,\frac{\max_i |c_i|}{\sigma},
\]
where \(C\) is the absolute constant in the Berry--Esseen theorem.
\end{theorem}

\begin{proof}
Since the \(x_i\) are i.i.d. symmetric random variables with \(\mathbb{E}[x_i]=0\) and \(\mathbb{E}[x_i^2]=1\), the random variables \(c_i x_i\) satisfy
\[
\mathbb{E}[c_i x_i] = 0,\quad \mathbb{E}[(c_i x_i)^2] = c_i^2,\quad \text{and} \quad \mathbb{E}[|c_i x_i|^3] = |c_i|^3.
\]
Define
\[
S_n = X = \sum_{i=1}^n c_i x_i,\quad \sigma^2 = \sum_{i=1}^n c_i^2,\quad \rho_n = \sum_{i=1}^n |c_i|^3.
\]
By the Berry--Esseen theorem we have, for all \(y\in\mathbb{R}\),
\[
\left| \Pr\Bigl[\frac{X}{\sigma}\le y\Bigr] - \Phi(y) \right| \le C\,\frac{\rho_n}{\sigma^3},
\]
where \(\Phi(y)\) is the standard normal cumulative distribution function.

Now, fix any \(t\in\mathbb{R}\) and let
\[
y_0 = \frac{t}{\sigma} \quad \text{and} \quad y_1 = \frac{t+\delta\sigma}{\sigma} = y_0+\delta.
\]
Then
\[
\Pr\Bigl[X\in [t,t+\delta\sigma]\Bigr] = \Pr\Bigl[\frac{X}{\sigma}\in [y_0, y_1]\Bigr] 
= \Pr\Bigl[\frac{X}{\sigma}\le y_1\Bigr] - \Pr\Bigl[\frac{X}{\sigma}\le y_0\Bigr].
\]
Using the Berry--Esseen bound for both endpoints, we have
\[
\Pr\Bigl[\frac{X}{\sigma}\le y_i\Bigr] = \Phi(y_i) + \varepsilon_i,\quad \text{with } |\varepsilon_i|\le C\,\frac{\rho_n}{\sigma^3},\quad i=0,1.
\]
Therefore,
\[
\Pr\Bigl[X\in [t,t+\delta\sigma]\Bigr] = \Phi(y_1)-\Phi(y_0) + (\varepsilon_1-\varepsilon_0).
\]
Taking absolute values and using the triangle inequality, we deduce
\[
\Pr\Bigl[X\in [t,t+\delta\sigma]\Bigr] \le \Phi(y_1)-\Phi(y_0) + 2C\,\frac{\rho_n}{\sigma^3}.
\]

Next, note that by the mean value theorem for the differentiable function \(\Phi\), there exists some \(y^*\in [y_0,y_1]\) such that
\[
\Phi(y_1)-\Phi(y_0) = \delta\,\varphi(y^*),
\]
where \(\varphi(y)=\frac{1}{\sqrt{2\pi}}e^{-y^2/2}\) is the standard normal density. Since \(\varphi(y)\le \frac{1}{\sqrt{2\pi}}\) for all \(y\), we obtain
\[
\Phi(y_1)-\Phi(y_0) \le \frac{\delta}{\sqrt{2\pi}}.
\]

It remains to control the term \(\rho_n/\sigma^3\). Since
\[
\rho_n = \sum_{i=1}^n |c_i|^3 \le (\max_i |c_i|) \sum_{i=1}^n c_i^2 = (\max_i |c_i|) \sigma^2,
\]
we have
\[
\frac{\rho_n}{\sigma^3} \le \frac{\max_i |c_i|}{\sigma}.
\]
Thus, the error term satisfies
\[
2C\,\frac{\rho_n}{\sigma^3} \le 2C\,\frac{\max_i |c_i|}{\sigma}.
\]

Combining the estimates, we conclude that
\[
\Pr\Bigl[X\in [t,t+\delta\sigma]\Bigr] \le \frac{\delta}{\sqrt{2\pi}} + 2C\,\frac{\max_i |c_i|}{\sigma}.
\]
This completes the proof.
\end{proof}


We will also need the standard Hoeffding's inequality. 

\begin{lemma}[Hoeffding's Inequality]
\label{lem:hoeffding}
Let \( X_1, X_2, \ldots, X_n \) be independent random variables such that \( a_i \leq X_i \leq b_i \) almost surely for each \( i \in \{1, 2, \ldots, n\} \). Define the sample mean 
\[
\overline{X} = \frac{1}{n} \sum_{i=1}^n X_i,
\]
and let \( \mu = \mathbb{E}[\overline{X}] \) be the expected value of the sample mean. Then for any \( t > 0 \),
\[
\mathbb{P}\left(|\overline{X} - \mu| \geq t \right) \leq 2 \exp\left(-\frac{2n^2 t^2}{\sum_{i=1}^n (b_i - a_i)^2}\right).
\]
\end{lemma}

The following application will be useful.

\begin{lemma}[Lower Bound on the Sum of Bernoulli Variables]
\label{lem:lower_bound_on_bernoulli_sum}
Let \(X_1, X_2, \dots, X_n\) be independent random variables such that for each \(i\),
\[
\Pr[X_i = 1] = p \quad \text{and} \quad \Pr[X_i = 0] = 1-p,
\]
with \(p \in (0,1]\). Then, for any \(\delta \in (0,1)\), if
\[
n \ge \frac{2\ln(1/\delta)}{p^2},
\]
it holds with probability at least \(1-\delta\) that
\[
\sum_{i=1}^n X_i \ge \frac{np}{2}.
\]
\end{lemma}

\begin{proof}
Since each \(X_i\) is a Bernoulli random variable taking values in \([0,1]\) with \(\mathbb{E}[X_i]=p\), we have
\[
\mathbb{E}\left[\sum_{i=1}^n X_i\right]= np.
\]
By Hoeffding's inequality, for any \(t > 0\) it holds that
\[
\Pr\!\Biggl[\sum_{i=1}^n X_i \le np - t\Biggr] \le \exp\!\Biggl(-\frac{2t^2}{n}\Biggr).
\]
Our goal is to ensure
\[
\sum_{i=1}^n X_i \ge \frac{np}{2},
\]
which is equivalent to having a deviation no more than 
\[
np - \frac{np}{2} = \frac{np}{2}.
\]
Thus, setting \(t = \frac{np}{2}\) in Hoeffding's inequality, we obtain
\[
\Pr\!\Biggl[\sum_{i=1}^n X_i \le \frac{np}{2}\Biggr] \le \exp\!\Biggl(-\frac{2\Bigl(\frac{np}{2}\Bigr)^2}{n}\Biggr)
= \exp\!\Biggl(-\frac{n p^2}{2}\Biggr).
\]
To guarantee that this probability is at most \(\delta\), we require
\[
\exp\!\Biggl(-\frac{n p^2}{2}\Biggr) \le \delta.
\]
Taking logarithms on both sides gives
\[
-\frac{n p^2}{2} \le \ln(\delta),
\]
which is equivalent to
$
n \ge \frac{2\ln(1/\delta)}{p^2}.
$
Thus, if 
$n \ge \frac{2\ln(1/\delta)}{p^2},$
then with probability at least \(1-\delta\) we have
\[
\sum_{i=1}^n X_i \ge \frac{np}{2}.
\]
This completes the proof.
\end{proof}

\section{Teacher Training}
\label{app:teacher_training}

In this section, we recall the teacher training analysis from \cite{panigrahi2024progressive}, which serves as a foundation for our approach. We first describe the first stage of training, where our goal is to ensure that the weights corresponding to the support of the target parity function are amplified relative to the others. Then, we outline the second stage, which leverages this weight separation to drive the teacher model toward a low-loss solution. This two-stage process is critical for setting up the conditions under which our student training analysis will later succeed.

In \Cref{subsec:teach_training_1}, we restate \Cref{lem:effect_of_1_step_on_w}, which shows that after one gradient descent step, 
the in-support weights become larger than the out-of-support weights. 
In \Cref{subsec:teach_training_2}, we recall \Cref{lem:second_stage_conclusion}, which establishes that if the conditions of 
\Cref{lem:effect_of_1_step_on_w} are met, 
the teacher can learn the top-level weights within $O(d^{O(k)} \eps^{-2} \log(dk/\eps \delta))$ samples, achieving a loss of at most $\eps$.

The teacher loss is given by a regularized version of the hinge loss $\ell_{f_t}(\bx, y)=\max(0, 1-f_t(\bx)y)$ for the first stage of training, and the standard hinge loss for the second stage.

\begin{algorithm}[H]
\caption{2-stage training for teacher}
\label{alg:2_stage_training}
\begin{algorithmic}[1]
\Require
Number of iterations $T_2$, 
Learning rates \(\eta_1, \eta_2\), 
batch sizes \(B_1, B_2\), 
weight decay \(\lambda_1\).\\
\textbf{Inner Layer Training:}
\For{\(t = 1\) 
}
    \State Sample \(B_1\)-samples \(\{(\bx^{(j)}, y^{(j)})\}_{j=1}^{B_1}\).
    \State Update the inner layer weights $\{ \bw_1^{(t)}, \dots, \bw_{m_t}^{(t)}\}$ as: 
    \[
    \bw_i^{(t)} \gets \bw_i^{(t-1)} - \eta_1 \mathbb{E}_{(\bx, y) \in \{(\bx^{(j)}, y^{(j)})\}_{j=1}^{B_1}} \left[ \nabla_{\bw_i} \left(\ell_{f_t}(\bx,y) + \lambda_1 \| \bw_{i}^{(t-1)}\|^2 \right) \right]
    \]
\EndFor\\
\textbf{Outer Layer Training:}
\For{\(t \in [0, T_2]\)}
    \State Sample \(B_2\)-samples \(\{(\bx^{(j)}, y^{(j)})\}_{j=1}^{B_2}\).
    \State Update the outer layer weights:
    \[
    \ba^{(t)} \gets \ba^{(t-1)} - \eta_2 \mathbb{E}_{(\bx, y) \in \{(\bx^{(i)}, y^{(i)})\}_{i=1}^{B_2}}\left[  \nabla_{\ba}\ell_{f_t}(\bx, y) \right]
    \]
\EndFor
\end{algorithmic}
\end{algorithm}

Our teacher is trained in exactly the same way as in \cite{panigrahi2024progressive}, using \Cref{alg:2_stage_training}.
For completeness, we recall conditions required for teacher training to succeed. Before we continue, we set up some notation.

\paragraph{Notation} 
\begin{itemize}
    \item In what follows, $B_1, B_2$ are batch sizes, i.e. number of samples drawn to estimate the gradients in the first and second stages of training respectively. $\delta$ will denote the probability of failure, and $\lambda_1$ will denote a regularization parameter (as seen in \Cref{alg:2_stage_training}). 
    \item $\tau_g$ is the error estimate of the gradient of the hinge loss, which in the case of the teacher turns out to be equivalent to the correlation loss at initialization, i.e. for a network $f$,  
$
\left| \mathbb{E}_{\bx, y \sim U(\{\pm 1\}^d)} \left[ \nabla_{w_{ij}} f(\bx)y \right] - \mathbb{E}_{\{(\bx_k, y_k)\}_{k=1}^{B_1}} \left[ \nabla_{w_{ij}} f(\bx)y \right] \right| \leq \tau_g,
$.
    \item $m_t$ and $m_s$ denote the hidden layer sizes of the student and teacher respectively.
    \item For $a, b, c \in \mathbb R$ we will say $a = b \pm c$ if $a \in [b-c, b+c]$. 
    \item For $u, v \in \R^d$, we define $u \odot v = (u_1v_1, \dots, u_dv_d)$.
    \item $\theta(t) := \{ \bW^{(t)}, \ba^{(t)}, \bb^{(t)}\}$ denotes the parameters of the model that the algorithm recovers at timestep $t$.
    \item $\maj(\bx) : \{\pm 1\}^d \rightarrow \{-1, 1\}$ returns $\text{sign}(\sum_{i=1}^d x_i)$, where $\text{sign}(0) :=1$ and $\text{sign}(t) := t/|t|$ for $t \neq 0$. 

\end{itemize}

We will also need the following lemma that controls the gradient error as a function of the batch size. 

\begin{claim}[Gradient Concentration \cite{panigrahi2024progressive}]
\label{claim:gradient_concentration}
Let $f$ be a two-layer network initialized using the symmetric initialization in \Cref{def:sym_initialization} with $m$ being its hidden dimension. 
Fix \(\delta, \tau_g > 0\). 
For all \(i \in [m], j \in [d]\), for a randomly sampled batch of size \(B_1\), \(\{(\bx_k, y_k)\}_{k=1}^{B_1}\), with probability at least \(1 - \delta\),
\[
\left| \mathbb{E}_{\bx, y \sim U(\{\pm 1\}^d)} \left[ \nabla_{w_{ij}} f(\bx) y\right] - \mathbb{E}_{\bx, y \sim \{(\bx_k, y_k)\}_{k=1}^{B_1}} \left[ \nabla_{w_{ij}} f(\bx) y \right] \right| \leq \tau_g,
\]
provided \(B_1 \geq \Omega\left(\tau_g^{-2} \log(md / \delta)\right)\).
\end{claim}




\subsection{Teacher analysis after first stage of training}
\label{subsec:teach_training_1}

After the first stage of teacher training, the weights satisfy the property that they have larger magnitudes for in-support indices compared to out-of-support indices. This property is crucial for the second stage of training to achieve a good solution. This behavior is formally captured by the following lemma. 

\begin{lemma}[Lemma B.2 (Single step gradient descent, from \cite{panigrahi2024progressive})] 
    \label{lem:effect_of_1_step_on_w}
    Let $\zeta_k$ denote the $k^{th}$ Fourier coefficient of the majority function. 
    Fix \( \tau_g, \delta > 0 \). Set \( T_1 = 1 \). 
    Suppose the batch size \( B_1 \geq \Omega(\tau_g^{-2} \log(m_td/\delta)) \). 
    For learning rate \( \eta_1 = \frac{m_t}{k |\zeta_{k-1}|} \) and \( \lambda_1 = 1/2\eta_1 \), 
    the following conditions hold true for all neurons \( i \in [m] \) at the end of the first stage of training with probability at least \( 1 - \delta \).
    \begin{enumerate}
        \item \( \left| w^{(1)}_{\ell j} - \frac{\operatorname{sign}(a^{(0)}_\ell \zeta_{k-1}) \operatorname{sign}(\chi_{[k] \setminus \{j\}}(w^{(0)}_\ell))}{2k} \right| \leq \frac{\tau_g}{|\zeta_{k-1}|}, \) for all \( j \in [k] \).
        
        \item \( \left| w^{(1)}_{\ell j} - \frac{\zeta_{k+1}}{|\zeta_{k-1}|} \frac{\operatorname{sign}(a^{(0)}_\ell) \operatorname{sign}(\chi_{[k] \cup \{j\}}(w^{(0)}_\ell))}{2k} \right| \leq \frac{\tau_g}{|k \zeta_{k-1}|}, \) for all \( j > k \).
    \end{enumerate}
\end{lemma}

While we do not reproduce the proof of \Cref{lem:effect_of_1_step_on_w}, we point out that the proof essentially follows by demonstrating that the gradients of the loss  $\nabla_{w_{\ell j}}[\ell_{f_t}(\bx, y)]$ at initialization satisfy properties similar to the ones stated above for $w_{\ell j}^{(1)}$, and setting \new{$\lambda_1 = 1/2\eta_1$} ensures that the weights after one step are proportional to these gradients $w_{\ell j}^{(1)} = -\eta_1 \nabla_{w_{\ell j}}[\ell_{f_t}(\bx, y)]$.

\paragraph{Teacher batch size:} A consequence of \Cref{lem:effect_of_1_step_on_w} is that, since we need the gap to be witnessed by the empirical gradients, the teacher batch size will be lower bounded by $B_1 \geq (d^2k ~\zeta_{k-1})^2 \log(m_td/\delta) \geq \Omega(d^{k-1})$. 
In fact this holds more generally that just for this algorithm (as shown in \citet{panigrahi2024progressive}).  

\subsection{Teacher analysis after second stage of training}
\label{subsec:teach_training_2}

Under the conclusions of \Cref{lem:effect_of_1_step_on_w}, the second stage of training produces a function with small loss relative to the unknown parity function, so long as the hidden layer is sufficiently large (i.e. $m_t \geq 2^k k \log(k/d)$). This result is formalized in the following theorem, which will be used to analyze the second stage of training for both the student and the teacher.


\begin{lemma}[Theorem 4, \cite{barak2022hidden}, version from \cite{panigrahi2024progressive}]
\label{lem:second_stage_conclusion}
Fix \( \epsilon, \delta > 0 \). Suppose 
\( m \geq \Omega(2^k k \log(k/\delta)) \), \( d \geq \Omega(k^4 \log(kd/\epsilon)) \). Furthermore, suppose \( B_1 \geq \Omega(d^{k} k^2 \log(kd/\epsilon)) \) such that the weights satisfy the conditions in \Cref{lem:effect_of_1_step_on_w} with \( \tau_g = O(d^{-k}k^{-1}d^{-2}) \) after the first phase, and let $\theta(t)$ denote the model at timestep $t$.
Then after 
\( T_2 = \Omega(md^2 k^3/\epsilon^2) \) steps of training with batch size \( B_2 = 1 \) and learning rate \( \eta_2 = 4k^{1.5}/(d m(T_2 - 1)) \),
we have, with expectation over the randomness of the initialization and the sampling of the batches:
\[
\min_{t \in [T_2]} \mathbb{E}[L_{\theta(t)}(x, y)] \leq \epsilon.
\]
Thus, the minimal sample complexity to reach a loss of \( \epsilon \) is given by:
\[
T_1 \times B_1 + T_2 \times B_2 = \Theta(d^{O(k)} \eps^{-2}\log(dk/\epsilon\delta)).
\]
\end{lemma}







\section{Student Training}
\label{app:student_training}

In this section, we analyze our student training algorithm (\Cref{alg:2_stage_training_student}). The student algorithm differs from the teacher training algorithm only in the first phase, where we use the distillation loss \(\ell_{DL}(\bx, f, g) := 
\new{-f(\bx) \cdot g(\bx)}
\). Here, \(f = f_s^{(1)} \in \R^{m_s}\) is the first layer of the student network, and \(g = Af_t^{(1)}(\bx)  \in \R^{m_s}\) is the first layer of the teacher network projected to the student's hidden layer dimension.

\begin{algorithm}[H]
\caption{2-stage training for student}
\label{alg:2_stage_training_student}
\begin{algorithmic}[1]
\Require 
Number of iterations $T_2$, 
Learning rates \(\eta_1, \eta_2\), 
batch sizes \(B_1, B_2\), 
weight decay \(\lambda_1\).\\
\textbf{Inner Layer Training:}
\For{\(t = 1\) 
}
    \State Sample \(B_1\)-samples \(\{(\bx^{(j)}, y^{(j)})\}_{j=1}^{B_1}\).
    \State Update the inner layer weights $\{ \bw_1^{(t)}, \dots, \bw_{m_s}^{(t)} \}$ as: 
    \[
    \bw_i^{(t)} \gets \bw_i^{(t-1)} - \eta_1 \mathbb{E}_{(\bx, y) \in \{(\bx^{(j)}, y^{(j)})\}_{j=1}^{B_1}} \left[ \nabla_{\bw_i} \left(\ell_{DL}(\bx,f_s^{(1)}, Af_t^{(1)}) + \lambda_1 \| \bw_{i}^{(t-1)}\|^2 \right) \right]
    \]
\EndFor\\
\textbf{Outer Layer Training:}
\For{\(t \in [0, T_2]\)
}
    \State Sample \(B_2\)-samples \(\{(\bx^{(j)}, y^{(j)})\}_{j=1}^{B_2}\).
    \State Update the outer layer weights:
    \[
    \ba^{(t)} \gets \ba^{(t-1)} - \eta_2 \mathbb{E}_{(\bx, y) \in \{(\bx^{(i)}, y^{(i)})\}_{i=1}^{B_2}}\left[  \nabla_{\ba} \ell(\bx, y) \right]
    \]
    
\EndFor
\end{algorithmic}
\end{algorithm}

\subsection{First stage analysis of the student}

Most of our effort will focus on showing that \(\bW^{(1)}_s\), the first layer of the student network after the first stage of training, satisfies a property similar to the conclusion of \Cref{lem:effect_of_1_step_on_w}. 

By choosing \(\lambda_1 = 1/(2\eta_1)\) in \Cref{alg:2_stage_training_student}, we obtain
$$\bw_i^{(1)} = -\eta_1 \mathbb{E}_{(\bx, y) \in \{(\bx^{(j)}, y^{(j)})\}_{j=1}^{B_1}} \left[ \nabla_{\bw_i} \ell_{DL}(\bx, f_s^{(1)}, Af_t^{(1)}) \right].$$
Thus, it suffices to show that the gradient update for the student has larger magnitudes for in-support coordinates than for out-of-support coordinates. This is captured in \Cref{lem:degree_1_gap,lem:expansion_of_f_r}, which will be the focus of this section.
We first recall some variants of lemmas from \cite{panigrahi2024progressive} which we will need.
\subsubsection{Preliminary Setup}

The following lemma shows that the gradient may be expressed as a function of the Fourier coefficients of $(Af_t^{(1)}(\bx))_i$ and $(Af_t^{(1)}(\bx))_i ~\maj(\bw_i \odot \bx)$.

\begin{lemma}[Teacher Correlation Gap implies Student Gradient Gap]
\label{lem:teach_gap_implies_student_gap}
Suppose 
\[
\left| \E_\bx \big[(Af_t^{(1)}(\bx))_i ~x_j\big] + \E_\bx \big[(Af_t^{(1)}(\bx))_i~\maj(\bw_i \odot \bx) x_j\big] \right| =
\begin{cases}
> \gamma_1, & \text{for } j \in [k], \\
< \gamma_2, & \text{for } j > k.
\end{cases}
\]
Then, 
\[
\left| \E_{\bx} \left[ \nabla_{w_{ij}} \ell_{DL}(\bx, f_s^{(1)}, Af_t^{(1)}) \right] \right| =
\begin{cases}
> \gamma_1/2, & \text{for } j \in [k], \\
< \gamma_2/2, & \text{otherwise.}
\end{cases}
\]
\end{lemma}

\begin{proof}
    At initialization, the gradient of the weight vector of neuron $i$ at coordinate $j$ is given by,
    \begin{align*}
        \E_{\bx}[\nabla_{w_{ij}} \ell_{DL}(\bx, f_s^{(1)}, Af_t^{(1)})] 
        &= -\E_{\bx}[\nabla_{w_{ij}}( f_s^{(1)} \cdot Af_t^{(1)})]\\
        &= -\E_x[\mathbf{1}(\bw_{i} \cdot \bx + b_i \geq 0) (Af_t^{(1)})_i x_j]
    \end{align*}
    Since $|b_i| < 1$  and $\bw_i, \bx \in \{\pm 1\}^d$ at initialization,
$\mathbf{1}(\bw_{i} \cdot \bx + b_i \geq 0) = \text{sign}(\bw_i \cdot \bx)$, since $b_i$ cannot contribute enough
to change the sign, and so, 
    $\mathbf{1}(\bw_{i} \cdot \bx + b_i \geq 0) = \frac 1 2 + \frac {\maj(\bw_i \odot \bx)}{2}$. 
    Substituting this above, 
    \begin{align*}
        \E_{\bx}[\nabla_{w_{ij}} \ell_{DL}(\bx, f_{s}^{(1)}(\bx), Af_t^{(1)}(\bx))] 
        = -\frac{1}{2} \left( \E_\bx[(Af_t^{(1)})_i(\bx) ~x_j] + \E_\bx [(Af_t^{(1)})_i(\bx) ~\maj(\bw_i \odot \bx)~x_j] \right).
    \end{align*}
\end{proof}

We will also need the following lemma, which provides bounds on $|\E_{\bx}[\phi_{b_\ell} (\bw_\ell \cdot \bx) x_j]|$ where $\phi_b(a) := \sigma(a+b)-\sigma(-a+b)$, when $\{\bw_1, \dots, \bw_{m_t}\}$ satisfy the conclusion of \Cref{lem:effect_of_1_step_on_w}. 
These bounds will later help us control the magnitude of the gradient in the in-support and out-of-support coordinates. The key idea is that due to concentration, a constant fraction of neurons exhibit in-support correlations of order of $\Omega(1/k)$ while the out-of-support correlations are of the order of $O(1/kd)$.

\begin{lemma}[Bounds on coefficients]
\label{lem:bounds_on_coeffs}
For a teacher network in the setting of \Cref{lem:second_stage_conclusion}; with probability $1-\delta$ over the randomness of initialization of $b_\ell$, the following hold as long as $m_t \geq 10 \log(1/\delta)$:
\begin{enumerate}
    \item \label{item:half_b_large} 
    For $j \in [k]$, there are at least $m_t/8$ values of $\ell \in [m_t/2]$ satisfying    
    $\left| \E_{\bx}\left[\phi_{b_\ell}\left(\bw_\ell \cdot \bx\right) x_j \right] \right| \geq \Omega(1/k)$, and for all $j\in[k]$ and $\ell \in [m_t/2]$, we have $\left| \E_{\bx}\left[\phi_{b_\ell}\left(\bw_\ell \cdot \bx\right) x_j \right] \right| \leq O(1/k)$. 
    
    \item \label{item:all_b_small} For all $j > k$ and $\ell \in [m_t/2]$,
    $\left|\E_{\bx}\left[\phi_{b_\ell}(\bw_\ell \cdot \bx) \right] \right| \leq O(1/kd)$.
\end{enumerate}
\end{lemma}
\begin{proof}
This result follows from the calculations in the ``estimates of in-support correlations'' and ``estimations of out-of-support correlations'' sections of Lemma B.5 in \cite{panigrahi2024progressive} (pages 25--26). 

\Cref{item:half_b_large} follows from the analysis in the ``estimates of in-support correlations'' section, which shows that with probability at least \(1/2\) over the randomness of $b_\ell$,
\[
\left| \mathbb{E}_{\bx}[\phi_{b_\ell}(\bw_\ell \cdot \bx) x_j] \right| \geq \frac{1}{4k} - O(\tau_g d |\zeta_{k-1}|^{-1}).
\]
Applying Hoeffding’s inequality to this event, we conclude that if \(m_t \geq \Omega(\log(1/\delta))\), then with probability \(1 - \delta\), at least \(m_t/8\) neurons satisfy
\[
\left| \mathbb{E}_{\bx}[\phi_{b_\ell}(\bw_\ell \cdot \bx) x_j] \right| \geq \frac{1}{16k} - O(\tau_g d |\zeta_{k-1}|^{-1}).
\]
Since $|w_{ij}| \leq 1/(2k) \pm (\tau_g/|\xi_{k-1}|)$ for \emph{all} $j \in [k]$ and $i \in [m_t]$, we can show an upper bound of $O(1/k)$ on this expectation from arguments similar to those used to establish \Cref{item:all_b_small} below.

\Cref{item:all_b_small} follows directly from the ``estimations of out-of-support correlations'' section on page 26 in \cite{panigrahi2024progressive}.

The error term \(2d\tau_g|\xi_{k-1}|^{-1}\) for the trained \emph{teacher network} is controlled by setting \(\tau_g\) appropriately. 
Note that this is not something that affects the student sample complexity, but only the teacher sample complexity. 
\end{proof}

\subsubsection{Teacher correlation Gap and Gradient Correlation Gap}

In this section we focus on establishing the hypothesis of \Cref{lem:teach_gap_implies_student_gap} -- i.e. a gap 
 between the in-support and out-of-support indices $j \in [d]$ for any fixed $i$ in the expression,
$\E_\bx[(Af_t^{(1)})_i(\bx) ~x_j] + \E_\bx [(Af_t^{(1)})_i(\bx) ~\maj(\bw_i \odot \bx)~x_j]$.

In \Cref{lem:degree_1_gap} we show this gap for the first term, $\E_\bx[(Af_t^{(1)})_i(\bx) ~x_j]$. 
In \Cref{lem:expansion_of_f_r} we show the second term $\E_\bx [(Af_t^{(1)})_i(\bx) ~\maj(\bw_i \odot \bx)~x_j]$ is dominated by the first term.

\begin{lemma}[Correlation Gap for Projected Teacher Dimensions]
\label{lem:degree_1_gap}
Let $(Af_t^{(1)})_\ell(\bx) := \sum_{i=1}^{m_t} a_i^{(\ell)} \sigma(\bw_i \cdot \bx+b_i)$
where $\ell \in [m_s]$ and $a_i^{(\ell)}$ are independently drawn u.a.r. from \new{$U(\{\pm 1/m_t\})$} for $i \in [m_t/2]$ and $a^\ell_{i+m_t/2} = -a^\ell_i$, and $\{ \bw_1, \dots \bw_{m_t} \}$ satisfy the conclusion of \Cref{lem:effect_of_1_step_on_w}.
Let $m_t \geq \Omega(k^2/\delta^2)$, then for a fixed $\ell$, with probability $1-\delta$
\[ \min_{j \in [k]} \left|\E_\bx[(Af_t^{(1)})_\ell(\bx) x_j]\right| > \new{\frac{\delta}{\sqrt{m_t} k^2}} \text{ and } \max_{j > k} \left|\E_\bx[(Af_t^{(1)})_\ell (\bx)x_j]\right| < \new{\frac{\log(m_sd/\delta)}{\sqrt{m_t} kd}}\]
 In particular, if $d > \Omega(k^4)$ then as long as the number of student neurons $m_s > T \log(1/\delta')$, with probability $1-\delta'$ this results in a gap larger than $\Omega(1/\sqrt{m_t} k^2)$ for some subset of $T$ dimensions.
\end{lemma}
\begin{proof}
For a fixed value of \(\ell\) (which we will later index over the output dimensions of the random projection), define
\[
f^\ell(\bx) := (Af_t^{(1)})_\ell(\bx) = \sum_{i=1}^{m_t} a_i^\ell \sigma(\bw_i \cdot \bx+b_i).
\]
It is convenient to rewrite the sum by grouping the terms corresponding to \(i\) and \(i+m_t/2\):
\begin{align*}
f^\ell(\bx) 
&= \sum_{i=1}^{m_t} a_i^\ell \sigma(\bw_i \cdot \bx+b_i)\\[1mm]
&= \sum_{i=1}^{m_t/2} a_i^\ell \Bigl(\sigma(\bw_i \cdot \bx+b_i) - \sigma(-\bw_i \cdot \bx+b_i)\Bigr).
\end{align*}
The second equality follows from the symmetry inherited from initialization (see \Cref{lem:effect_of_1_step_on_w} and the discussion in \cite{panigrahi2024progressive}). For convenience, define
\[
\phi_{b_i}(a) := \sigma(a+b_i) - \sigma(-a+b_i).
\]
Then, by linearity of expectation (after multiplying by \(x_j\)), we have
\[
\E_\bx[f^\ell(\bx)x_j] = \sum_{i=1}^{m_t/2} a_i^\ell\, \E_\bx\Bigl[\phi_{b_i}(\bw_i \cdot \bx)x_j\Bigr].
\]


For the remainder of the proof, recall that the teacher weights and biases \(\{\bw_i,b_i\}_{i\in[m_t]}\), are fixed after the first stage of training. Thus, the only randomness comes from the independent choices of the coefficients \(a_i^\ell\). Under this conditioning, for each fixed \(j\) we define the sum
\[
S^\ell_j = \sum_{i=1}^{m_t/2} a_i^\ell\, \E_\bx\Bigl[\phi_{b_i}(\bw_i \cdot \bx)x_j\Bigr].
\]
Since the \(a_i^\ell\) are sampled independently across \(\ell\), the sums \(S^\ell_j\) are independent across different student indices \(\ell\).
For the next two calculations (anti-concentration for in-support dimensions, and concentration for out-of-support dimensions) assume that $\ell$ is fixed.

\medskip
\textbf{Anti-Concentration for In-Support Dimensions:}\\
To lower-bound \(|\E_\bx[f^\ell(\bx)x_j]|\) for \(j\in[k]\), we rely on anti-concentration inequalities for sums of Rademacher random variables. By \Cref{lem:bounds_on_coeffs}, for each \(j\in[k]\) at least \(m_t/16\) of the indices \(i\) satisfy
\[
\E_\bx\Bigl[\phi_{b_i}(\bw_i \cdot \bx)x_j\Bigr] \ge \Omega(1/k).
\]
And all of the indices $i$ satisfy $\E_\bx\Bigl[\phi_{b_i}(\bw_i \cdot \bx)x_j\Bigr] \le O(1/k)$. Recall that each \(a_i^\ell\) is in \(\{+1/m_t,-1/m_t\}\). Thus, for \(i\) such that the lower bound stated above holds, the contribution is
$
a_i^\ell\, \E_\bx\Bigl[\phi_{b_i}(\bw_i \cdot \bx)x_j\Bigr] \ge \frac{\Omega(1/k)}{m_t}.
$
Consequently, the variance of \(S^\ell_j\) is
\[
\sigma^2 = \sum_{i=1}^{m_t/2} \Bigl(a_i^\ell\, \E_\bx[\phi_{b_i}(\bw_i \cdot \bx)x_j]\Bigr)^2 
= \sum_{i=1}^{m_t/2} \left(\frac{1}{m_t}\right)^2 \E_\bx\Bigl[\phi_{b_i}(\bw_i \cdot \bx)x_j\Bigr]^2.
\]
For the at-least \(m_t/16\) indices with \(\E_\bx[\phi_{b_i}(\bw_i \cdot \bx)x_j] \ge \Omega(1/k)\), we have each contributing at least \(\Omega\left((1/m_t)^2\cdot(1/k^2)\right)\). Hence,
$
\sigma^2 \ge \frac{m_t/16}{m_t^2}\cdot \Omega\Bigl(\frac{1}{k^2}\Bigr) = \Omega\Bigl(\frac{1}{m_t\, k^2}\Bigr).
$
This yields a standard deviation 
$
\sigma \ge \Omega\Bigl(\frac{1}{\sqrt{m_t}\, k}\Bigr).
$
Now, applying the anti-concentration inequality for Rademacher sums (see \Cref{thm:anticoncentration-rademacher}), for any fixed \(j\in[k]\) and $\ell \in [m_s]$ we have
\[
\Pr_{\{a_i^\ell \}_{i
}}\Biggl[\Bigl|S^\ell_j\Bigr| \le \frac{\delta}{\sqrt{m_t}k^2}\Biggr] \le \frac{\delta}{2k} + O\Bigl(\frac{1/(m_t k)}{1/(\sqrt m_t k)}\Bigr) = \frac{\delta}{2k} + O\Bigl(\frac{1}{\sqrt m_t}\Bigr).
\]
Taking a union bound over the \(k\) coordinates in \([k]\) (which results in a multiplication by $k$ on the right hand side) and choosing \(m_t \ge \Omega(k^2/\delta^2)\) so that the error term \(O(k/\sqrt{m_t})\) is at most \(\delta/2\), we obtain
\begin{equation}
\label{eqn:penultimate_lower}
\Pr_{\{a_i^\ell\}_i} \Biggl[ \exists\, j \in [k] \text{ such that } \Bigl|S^\ell_j\Bigr| \le \frac{\delta}{\sqrt{m_t}k^2}\Biggr] \le k \cdot \left( \frac \delta {2k} + O\left( \frac 1 {\sqrt m_t} \right) \right)  <\delta.
\end{equation}

\medskip
\textbf{Concentration for Out-of-Support Dimensions:}\\
For coordinates \(j > k\), \Cref{lem:bounds_on_coeffs} (specifically, \Cref{item:all_b_small}) implies that the coefficients \(\E_\bx[\phi_{b_i}(\bw_i \cdot \bx)x_j]\) are uniformly small. An application of Hoeffding's inequality shows that
\begin{equation}
\label{eqn:penultimate_upper}
\Pr_{\{a_i^\ell\}_i} \Biggl[\exists\, j > k \text{ such that } \Bigl|S^\ell_j\Bigr| > \Omega\Bigl(\frac{\log(d/\delta)}{\sqrt{m_t}kd}\Bigr)\Biggr] \le \delta.
\end{equation}

Here the logarithmic factor arises naturally from applying a union bound over the at most \(d-k\) out-of-support coordinates.

\medskip
\textbf{Defining “Good” Projected Dimensions and Aggregating over Student Neurons:}\\
We define a projected dimension (i.e. a particular output coordinate indexed by \(\ell\)) to be “good” if it satisfies both \eqref{eqn:penultimate_lower} and the \eqref{eqn:penultimate_upper}. Since the teacher parameters are fixed, the only randomness is over the independent coefficients \(a_i^\ell\). Hence, the events that different dimensions are good are independent. 

In particular, if we set \(\delta = 1/8\) then the probability that a given dimension is good is at least \(1/4\). By applying a standard lower-tail bound for sums of independent Bernoulli random variables (see \Cref{lem:lower_bound_on_bernoulli_sum}), we conclude that with probability at least \(1-\delta'\) there are at least 
$T$ good projected dimensions, provided that
$m_s \ge \Omega\bigl(T\,\log(1/\delta')\bigr).$

\medskip
\textbf{Gap Between In-Support and Out-of-Support Dimensions:}\\
With \(\delta = 1/8\), for \(j\in[k]\) the magnitude \(|\E_\bx[f^\ell(\bx)x_j]|\) is at least \(\Omega(1/(\sqrt{m_t}k^2))\) while for \(j>k\) it is at most \(O(\log(d)/(\sqrt{m_t}kd))\). Note that if 
\[
\frac{\log(d)}{d} \ll \frac{1}{k^2},
\]
which is ensured when \(d > \Omega(k^4)\) (since for \(d \ge Ck^4\) one has \(\frac{\log(d)}{d} \le \frac{\log(Ck^4)}{Ck^4} \ll \frac{1}{k^2}\)), then the gap between in-support and out-of-support dimensions is at least
\[
\Omega\Biggl(\frac{1}{\sqrt{m_t}}\Bigl(\frac{1}{k^2} - \frac{\log(d)}{d}\Bigr)\Biggr) = \Omega\Bigl(\frac{1}{\sqrt{m_t}k^2}\Bigr).
\]

This completes the proof that, under the stated conditions on \(m_t\), \(m_s\), and \(d\), with high probability there exists a subset of \(T\) good projected dimensions for which the desired correlation gap holds.
\end{proof}

We now focus on getting similar bounds on $\mathbb{E}_\bx [f^\ell(\bx) \cdot \maj(\bw \odot \bx) x_j]$.
To bound this term, we must first bound all the Fourier coefficients (\Cref{def:fourier_expansion}) of \(f^\ell(\bx)\). This is necessary because the degree-1 Fourier coefficients of a product of Boolean functions (in this case $f^\ell(\bx)$ and $\maj(\bw \odot \bx)$) depend on their \emph{entire} Fourier expansions (see \Cref{lem:fourier_inner_product} below). 


\begin{lemma}[Fourier Coefficients of Inner Product]
\label{lem:fourier_inner_product}
Let \(f, g: \{-1, 1\}^n \to \mathbb{R}\) be two Boolean functions with Fourier expansions:
\[
f(x) = \sum_{S \subseteq [n]} \hat{f}(S) \chi_S(x) \quad \text{and} \quad g(x) = \sum_{T \subseteq [n]} \hat{g}(T) \chi_T(x),
\]
where \(\chi_S(x) = \prod_{i \in S} x_i\) are the parity (Walsh) basis functions, and \(\hat{f}(S)\), \(\hat{g}(T)\) are the Fourier coefficients of \(f\) and \(g\), respectively.

Then, the Fourier coefficients of the inner product \(h(x) = f(x) \cdot g(x)\) are given by:
\[
\hat{h}(S) = \sum_{T \subseteq [n]} \hat{f}(T) \hat{g}(S \triangle T),
\]
where \(S \triangle T\) denotes the symmetric difference of the sets \(S\) and \(T\).
\end{lemma}

The bounds on the Fourier coefficients of the expansion of $(Af_t^{(1)})$ follow from the following modified versions of Lemma B.5 and Corollary B.6 from \cite{panigrahi2024progressive}, which we state below.

\begin{lemma}[Correlation within-support variables]
\label{lem:in_support_corr_fancy}
Under the event that the conditions in \Cref{lem:effect_of_1_step_on_w} are satisfied by each neuron, which occurs with probability at least 
$99\%$
w.r.t. the randomness of initialization as long as $m_t \geq \Omega(k^2)$ and $d \geq \Omega(k^4)$, the output of the \emph{teacher} network after the first phase satisfies the following conditions for all $i \in [m_s]$:
\begin{enumerate}
    \item $\mathbb{E}_{\bx,y} \left[ (Af^{(1)}_t)_i(\bx) x_j \right] \geq \Omega(\frac{1}{\sqrt{m_t} k^2}) 
    $ for all $j \in S$.
    \item $\mathbb{E}_{\bx,y} \left[ (Af^{(1)}_t)_i(\bx) x_j \right] \leq \tilde O\left(\frac{1}{\sqrt{m_t}d}\right)$ for all $j \notin S$.
    \item $\mathbb{E}_{\bx,y} \left[ (Af^{(1)}_t)_i(\bx) \chi_S(\bx) \right] \leq O\left(\tau_g d |\zeta_{k-1}|^{-1} 
    \right)$ for all $S$ with even $|S|$.
    \item $\left\| (Af^{(1)}_t)_i(\bx) \right\|_2^2 = \mathbb{E}_{\bx,y} \left[ (Af^{(1)}_t)_i(\bx) \right]^2 \leq O\left(
    \frac{d}{k}\right)$. 
\end{enumerate}
\end{lemma}
\begin{proof}
The first two items follow from \Cref{lem:degree_1_gap}.
The proofs of the second two items are exactly the same as the proofs of Items 3 and 4 of Lemma B.5 in \cite{panigrahi2024progressive}.
\end{proof}

\Cref{lem:in_support_corr_fancy} now allows us to recover the following variant of Corollary B.6 from \cite{panigrahi2024progressive}, 
effectively achieving an upper bound on $|\E_{\bx, y}[f^\ell(\bx) \maj(\bw \odot \bx)x_j]|$. 


\begin{corollary}[Bound on Correlation of $f^\ell(\bx) \maj(\bw \odot \bx)$ with $x_j$]
\label{lem:expansion_of_f_r}
Let $f^\ell(\bx)$ be defined as in \Cref{lem:bounds_on_coeffs} and suppose the conditions in 
\Cref{lem:bounds_on_coeffs} are satisfied by each neuron, 
which occurs with probability at least \(99\%\) with respect to the randomness of initialization and sampling, 
the output of the model after the first phase can be given as:
\begin{align*}
f^{\ell}(\bx) 
&= \sum_{j=1}^k c_j x_j 
+ \sum_{j=k+1}^d c_j x_j + \sum_{\substack{S \subseteq [d] \\ |S| \% 2 = 1, |S| \geq 3}} c_S \chi_S(x) 
+ \sum_{\substack{S \subseteq [d] \\ |S| \% 2 = 0}} c_S \chi_S(x),
\end{align*}
where
\begin{align*}
|c_j| &\geq \Omega((k^2 \sqrt{m_t})^{-1}), &\text{for all } 1 \leq j \leq k, \\
|c_j| &\leq O((k^3 \sqrt{m_t})^{-1}), &\text{for all } j > k, \\
|c_S| &\leq O(\tau_g d |\zeta_{k-1}|^{-1}), &\text{for all } S \subseteq [d] \text{ with } |S| \% 2 = 0, \\
|c_S| &\leq O(
d / k), &\text{for all } S \subseteq [d] \text{ with } |S| \% 2 = 1.
\end{align*}
As such, given a fixed $\bw$, the following correlations hold true for all \(i\):
\[
\mathbb{E}_{x,y}\left[f^\ell(\bx) \text{Maj}(\bw \odot \bx) x_i\right] = 
O\left(\frac{c_i}{\sqrt{d}} + 
\tau_g d^{5/3} |\zeta_{k-1}|^{-1}\right).
\]


If the \emph{teacher batch size} \(B_1\) is set such that \(\tau_g \leq O((k^3 \sqrt{m_t})^{-1} d^{-5/3} |\zeta_{k-1}|
)\), i.e. \(B_1 \geq \Omega(k^4  m_t ~ d^{10/3} \zeta_{k-1}^{-2})\), then for all \(i\), we can have 
$\mathbb{E}_{\bx,y}\left[f^\ell(\bx) \text{Maj}(\bw \odot \bx) x_j\right] \leq c_i/10,$
And hence, 
\begin{align*}
\left|\E_{\bx,y} \left[ \nabla_{w_{ij}} \ell_{DL}(\bx, f_s^{(1)}(\bx), Af_t^{(1)}(\bx)) \right]\right| = \Theta(|c_j|).
\end{align*}

\end{corollary}

\begin{proof}
The first set of bounds on $|c_j|$ follow from \Cref{lem:bounds_on_coeffs}. 
We now estimate the correlation of $f^{\ell}(\bx) \maj(\bw \odot \bx)$ with $x_i$.
Observe that for a fixed $\bw$, 
\begin{align*}
\mathbb{E}_{\bx,y}\left[f^\ell(\bx) \cdot \text{Maj}(\bw \odot \bx) x_i\right]
&= \E_{\bx, y} \left[ \sum_{j=1}^d c_jx_j\maj(\bw \odot \bx)x_i\right] + \sum_{S \subset [d], |S|\%2 = 1, |S| \geq 3} \E_{\bx,y} \left[ c_S \maj(\bw \odot \bx) \chi_S(\bx) x_i\right] \\
&\qquad + \E_{\bx, y} \left[ \sum_{S \subset[d], |S|\%2 = 0} c_s \maj(\bw \odot \bx) \chi_S(\bx) x_i\right].
\end{align*}
Since $\maj(\bw \odot \bx)$ is an odd function (for a fixed $\bw$), $\E_{\bx, y}[\maj(\bw \odot \bx) \chi_S(\bx)x_i] = 0$ for $|S|\%2 = 1$. This allows us to remove the term. A similar argument holds for the first term, giving us
\begin{align*}
\mathbb{E}_{\bx,y}\left[f^\ell(\bx) \cdot \text{Maj}(\bw \odot \bx) x_i\right]
&= c_i \E_{\bx, y} \left[\maj(\bw \odot \bx)\right] + \E_{\bx, y} \left[ \sum_{S \subset[d], |S|\%2 = 0} c_s \maj(\bw \odot \bx) \chi_S(\bx) x_i\right].
\end{align*}
Note that $\maj(\bw \odot \bx) = - \maj(\bw \odot (-\bx))$, and so the only terms that remain in the expectation $\E_{\bx, y}[\maj(\bw \odot \bx)]$ are those for which $\sum_{i=1}^d w_i x_i = 0$. Since $\text{sign}(0)=1$, this implies, $\E_{\bx, y}[\maj(\bw \odot \bx)] = \Pr_{\bx}[\sum_{i=1}^d w_i x_i = 0]$  which is either $0$ for even $d$ or $\Theta(1/\sqrt{d})$ for odd $d$. The second term may be bounded as follows, 
\begin{align*}
&\left| \E_{\bx, y} \sum_{S \subset[d], S\%2 = 0} c_S \maj(\bw \odot \bx)\chi_S(\bx) x_i \right|\\
&\leq O(\tau_g d |\zeta_{k-1}|^{-1} 
) \cdot \left( \sum_{S \subset[d], S\%2 = 0} |\E_{\bx, y} \maj(\bw \odot \bx) \chi_S(\bx) x_i|\right)\\
&\leq O(\tau_g d |\zeta_{k-1}|^{-1} 
) \cdot \left( \sum_{S \subset[d], S\%2 = 0} |\E_{\bx, y} \maj(\bw \odot \bx) \chi_S(\bx)|\right)\\
&\leq O(\tau_g d |\zeta_{k-1}|^{-1} 
) \cdot \left( \sum_{S \subset[d], S\%2 = 0} \Theta \left( \frac{|S|^{-1/3}}{\binom{d}{|S|}}\right)\right)\\
&\leq O(\tau_g d^{5/3} |\zeta_{k-1}|^{-1} 
)\\
\end{align*}
Where the bounds follow from standard bounds on the Fourier coefficients of the majority function. 
By ensuring that the batch size $B_1 \geq \tilde \Omega(\tau_g^{-2})$ 
, we see that for $\tau_g \leq O((k^3 \sqrt{m_t})^{-1} d^{-5/3} |\zeta_{k-1}|
)$ 
the contribution of the majority term to the gradient is small enough to not significantly affect the overall gap.
\end{proof}


We now see that this gap in the gradient values for the first stage of student training is transferred to the empirical gradient with probability $1-\delta$ using only $O((kd)^2\log(m_t d/\delta))$ samples.

\begin{corollary}[Empirical Gradient Gap]
\label{lem:emp_grad_gap}
Fix $\delta, \tau > 0$. Consider a randomly sampled batch of samples $B = \{(\bx_t, y_t)\}_{t=1}^{|B|}$
with size 
$|B| \ge \Omega((kd)^2 \log(m_t\,d/\delta)).$
Then, with probability at least $1-\delta$, for every index $i \in [m_s]$ and every coordinate $j \in [d]$, we have the empirical gradient satisfies
\[
\widehat \E_{\bx \sim B}\left[\nabla_{w_{ij}} \ell_{DL}\bigl(\bx, f_{s}^{(1)}(\bx), Af_t^{(1)}(\bx)\bigr)\right] =
\begin{cases}
\Omega\left(\frac{1}{\sqrt{m_t} k^2}\right) & \text{if } j \in [k], \\
\tilde O\left(\frac{1}{\sqrt{m_t} d}\right) & \text{if } j \notin [k].
\end{cases}
\]

\end{corollary}
\begin{proof}
Recall that from calculations similar to those in \Cref{lem:teach_gap_implies_student_gap},
\begin{align*}
    &\widehat \E_{\bx \sim B}[\nabla_{w_{ij}} \ell_{DL}(\bx, f_{s}^{(1)}(\bx), Af_t^{(1)}(\bx))]\\
    &= -\frac{1}{2} \left( \widehat \E_{\bx \sim B}[(Af_t^{(1)})_i(\bx) ~x_j] + \widehat \E_{\bx \sim B}[(Af_t^{(1)})_i(\bx) ~\maj(\bw_i \odot \bx)~x_j] \right)
\end{align*}
So to estimate the number of samples to achieve a gradient gap, it is sufficient to estimate the number of samples required to estimate the gap in the expectations on the right hand side.
We show that in $|B|$ samples, the first term is close enough to the true expectation; the argument for the second term is analogous. 

Define $(Af_t^{(1)})_i = \sum_{\ell=1}^{m_t} a_\ell^{(i)} \sigma(\bw_\ell \cdot \bx + b_\ell)$. Then we have \[ \widehat \E_{\bx \sim B}[(Af_t^{(1)})_i(\bx) \, x_j] = \sum_{\ell=1}^{m_t/2} a_\ell^{(i)}\, \widehat \E_{\bx \sim B}[\phi_{b_\ell}(\bw_\ell \cdot \bx) \, x_j].\] 
We will prove that if $|B| \ge \Omega((kd)^2 \log(1/\delta))$, then with probability at least $1-\delta$, 
$$
\left|\widehat \E_{\bx \sim B}[\phi_{b_\ell}(\bw_\ell \cdot \bx) \, x_j] - \E_{\bx \sim \{\pm 1\}^n}[\phi_{b_\ell}(\bw_\ell \cdot \bx) \, x_j]\right| \le O\left(\frac{1}{kd}\right).
$$ 
This shows that the properties of $\E_{\bx \sim \{\pm 1\}^n}[\phi_{b_\ell}(\bw_\ell \cdot \bx) \, x_j]$ described in \Cref{lem:bounds_on_coeffs} also apply to the empirical estimate -- specifically that for at least $m_t/16$ values of $\ell$, $\widehat \E_{\bx \sim B}[\phi_{b_\ell}(\bw_\ell \cdot \bx) \, x_j] \geq \Omega(1/k)$ for all in-support $j$; and for all $m_t$ values of $\ell$, $\widehat \E_{\bx \sim B}[\phi_{b_\ell}(\bw_\ell \cdot \bx) \, x_j] \leq O(1/kd)$ for out-of-support $j$, and $\widehat \E_{\bx \sim B}[\phi_{b_\ell}(\bw_\ell \cdot \bx) \, x_j] \leq O(1/k)$.

Therefore, the gap established in \Cref{lem:degree_1_gap} for the population expectation carries over to the empirical estimate.

We now apply Hoeffding's inequality to show concentration for $\widehat \E_{\bx \sim B}[\phi_{b_\ell}(\bw_\ell \cdot \bx) x_j]$. Since the weights $\bw_\ell$ come from after the first stage of teacher training and thus satisfy \Cref{lem:effect_of_1_step_on_w} --- the in-support weights are of size $O(1/k)$ and the out-of-support weights are of size $O(1/kd)$ --- the random variable $\phi_{b_\ell}(\bw_\ell \cdot \bx) x_j$ is bounded, with $|\phi_{b_\ell}(\bw_\ell \cdot \bx) x_j| \le O(1)$. We can then apply Hoeffding's inequality (\Cref{lem:hoeffding}) to obtain that, with probability at least $1 - 2\exp(-c|B|\tau_s^2)$ for some constant $c$, 
$$
\left|\E_{\bx}[\phi_{b_\ell}(\bw_\ell \cdot \bx)x_j] - \widehat \E_{\bx \sim B}[\phi_{b_\ell}(\bw_\ell \cdot \bx)x_j]\right| \le \tau_s.
$$
We only require that $\tau_s$ is smaller than $O(1/kd)$ to recover the bounds in \Cref{lem:bounds_on_coeffs}. The same argument applies to the second term.
Taking a union bound over all the $m_t d$ variables $w_{\ell j}$ and the two cases and putting everything together gives us the sample complexity bound for $|B|$.
\end{proof}

\subsection{Student sample complexity}

We now show that the student network requires fewer samples to recover a good solution. 



\begin{lemma}[Student version of Theorem 4, \cite{barak2022hidden}]
\label{lem:second_stage_conclusion_student}
Fix $\eps > 0$. Let $d \geq \Omega(k^4 \log(kd/\eps))$ and suppose a teacher $f_t$ with hidden dimension 
$m_t \geq m_s \geq \Omega(2^k k \log(k))$ 
is trained in the setting of \Cref{lem:second_stage_conclusion}.
Furthermore, for student training \Cref{alg:2_stage_training_student} \( B_1 \geq (kd)^2 \log(m_t d/\delta) \), \new{$\eta_1 = \sqrt{m_t} $}. Then with probability $99\%$, after 
\( T_2 = \Omega(m_sd^2 k^3/\epsilon^2) \) steps of training with batch size \( B_2 = 1 \) and learning rate \( \eta_2 = 4k^{1.5}/(d m_s(T_2 - 1)) \), we have, with expectation over the randomness of the initialization and the sampling of the batches:
\[
\min_{t \in [T_2]} \mathbb{E}[L_{\theta(t)}(\bx, y)] \leq \epsilon.
\]
Thus, the minimal sample complexity to reach a loss of \( \epsilon \) for the student is given by:
\[
T_1 \times B_1 + T_2 \times B_2 
= \Theta(2^{O(k)} d^2 \epsilon^{-2} \log(m_t k/\delta\epsilon)).
\]
\end{lemma}
\begin{proof}
For the first stage, the sample complexity of the student is determined by the sample complexity required to ensure that the gradient witnesses a gap between the in-support and out-of-support coordinates. \Cref{lem:emp_grad_gap} shows that with a batch $B$ of $O((kd)^2 \log(m_t d/\delta))$ samples, the expected gradient has in-support coordinates of $\Omega(1/\sqrt{m_t}k^2)$ and out-of-support coordinates of $O(1/\sqrt{m_t}d)$. 

After one step of gradient descent with an appropriate regularization parameter $(\lambda_1 = 1/2\eta)$, we have $\bw_i = -\eta \widehat \E_{(\bx, y)\sim B} \left[\nabla_{\bw_i} \ell_{DL} (\bx, f_s^{(1)}, Af_t^{(1)})\right]$.

Since $d \geq \Omega(k^4)$, we may then set $\eta = \sqrt{m_t}$ and $T = \Theta(k 2^k \log(k/\eps))$ and $m_s = O(T)$ to ensure that with probability $99\%$, the gap between at least $T$ randomly selected student weights out of a total of $m_s$ student weights matches the gap in \cite{panigrahi2024progressive} and is bounded below by $\Omega(1/k^2)$. 

This sets up the $T$ weights exhibiting a gap in the bottom layer to satisfy the same properties as in \cite{panigrahi2024efficient}. When restricted to the $T$ student weights that have a gap, the analysis of the second stage is the same as in their paper (and the same as in \Cref{lem:second_stage_conclusion}) -- in the event that we have some spurious neurons that don't exhibit a gap, the top layer can learn to ignore these by setting the contribution of those to be $0$. The exponential dependence in $k$ arises from the requirement in \cite{panigrahi2024efficient, barak2022hidden}.

\end{proof}
\emph{Remark:} 
We observe that \emph{even when the gap between the teacher's width and the student's width is only polynomial}, the teacher requires $\Omega(d^{k-1})$ samples, while the student only needs $\tilde{O}(2^{O(k)} \poly(d,k))$ samples.
This difference arises because of the difference in the magnitude of the gap between the in-support and out-of-support coordinates of the gradient in these two cases.

\section{Probabilistic Context-Free Grammars}
\label{app:pcfg}

In this section we formally define a PCFG.

\begin{definition}[Probabilistic Context-Free Grammar (PCFG)]
\label{def:PCFG}
A Probabilistic Context-Free Grammar (PCFG) is a 5-tuple $(N, \Sigma, S, R, P)$ where:
\begin{itemize}
    \item $N$ is a finite set of non-terminal symbols
    \item $\Sigma$ is a finite set of terminal symbols $(N \cap \Sigma = \emptyset)$
    \item $S \in N$ is the distinguished start symbol
    \item $R \subseteq N \times (N \cup \Sigma)^*$ is a finite set of production rules
    \item $P: R \to [0,1]$ is a probability function satisfying:
    \[
        \forall A \in N,\ \sum_{\substack{(A \to \beta) \in R}} P(A \to \beta) = 1
    \]
\end{itemize}
The probability of a derivation tree $T$ is given by:
\[
P(T) = \prod_{(A \to \beta) \in T} P(A \to \beta)
\]
\end{definition}

In \texttt{cfg3b}, the PCFG is constructed such that the degree for every non-terminal \( A \) is 2. 
In any generation rule, consecutive pairs of symbols in the generated strings are distinct. 
The $25\%, 50\%, 75\%$, and $95\%$ percentile string lengths generated by the PCFG are $251, 278, 308, and 342$, respectively, 
we refer to the commonly cited \Cref{fig:pcfg_indented} below from \cite{allen2023physics}.

\begin{figure}[ht]
\centering
\scalebox{0.8}{\ttfamily  
\setlength{\tabcolsep}{0pt} 
\begin{tabular}{l}
22 $\rightarrow$ 21 20 \\
22 $\rightarrow$ 20 19 \\
\hspace{1cm}19 $\rightarrow$ 16 17 18 \\
\hspace{1cm}19 $\rightarrow$ 17 18 16 \\
\hspace{1cm}20 $\rightarrow$ 17 16 18 \\
\hspace{1cm}20 $\rightarrow$ 16 17 \\
\hspace{1cm}21 $\rightarrow$ 18 16 \\
\hspace{1cm}21 $\rightarrow$ 16 18 17 \\
\hspace{2cm}16 $\rightarrow$ 15 13 \\
\hspace{2cm}16 $\rightarrow$ 13 15 14 \\
\hspace{2cm}17 $\rightarrow$ 14 13 15 \\
\hspace{2cm}17 $\rightarrow$ 15 13 14 \\
\hspace{2cm}18 $\rightarrow$ 15 14 13 \\
\hspace{2cm}18 $\rightarrow$ 14 13 \\
\hspace{3cm}13 $\rightarrow$ 11 12 \\
\hspace{3cm}13 $\rightarrow$ 12 11 \\
\hspace{3cm}14 $\rightarrow$ 11 10 12 \\
\hspace{3cm}14 $\rightarrow$ 10 11 12 \\
\hspace{3cm}15 $\rightarrow$ 12 11 10 \\
\hspace{3cm}15 $\rightarrow$ 11 12 10 \\
\hspace{4cm}10 $\rightarrow$ 7 9 8 \\
\hspace{4cm}10 $\rightarrow$ 9 8 7 \\
\hspace{4cm}11 $\rightarrow$ 8 7 9 \\
\hspace{4cm}11 $\rightarrow$ 7 8 9 \\
\hspace{4cm}12 $\rightarrow$ 8 9 7 \\
\hspace{4cm}12 $\rightarrow$ 9 7 8 \\
\hspace{5cm}7 $\rightarrow$ 3 1 \\
\hspace{5cm}7 $\rightarrow$ 1 2 3 \\
\hspace{5cm}8 $\rightarrow$ 3 2 \\
\hspace{5cm}8 $\rightarrow$ 3 1 2 \\
\hspace{5cm}9 $\rightarrow$ 3 2 1 \\
\hspace{5cm}9 $\rightarrow$ 2 1 \\
\end{tabular}}
\caption{\textbf{cfg3b} from \cite{allen2023physics}. Vocabulary is $\{1,2,3\}$. Indentation reflects production hierarchy.}
\label{fig:pcfg_indented}
\end{figure}

We also we show similar performance gains to those we observe in \Cref{sec:bert} for experiments with larger bandwidth. 
In particular, for experiments with a total of 6000 and 8000 iterations respectively, with three and four-stage curricula.

\section{Miscellaneous Figures}

\begin{figure*}[t]
    \centering
       \begin{minipage}[t]{0.5\linewidth}
        \centering
        \includegraphics[width=1\textwidth]{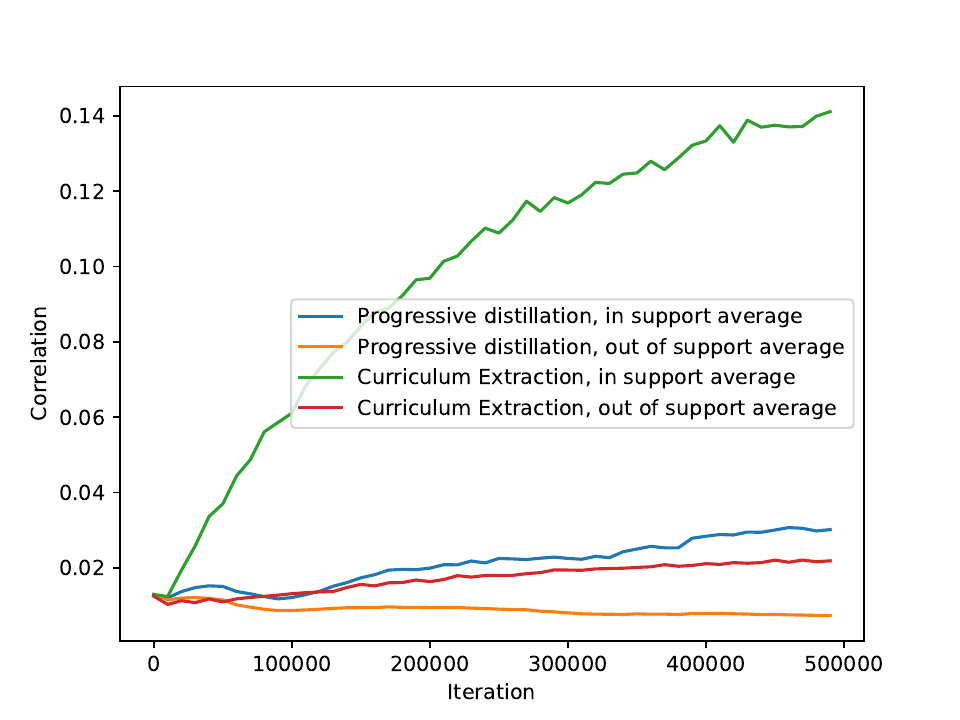}
        \vspace{0.5em} 
        \label{fig:mlp_prog_vs_layer_correlation}
    \end{minipage}
     \caption{%
    \textbf{MLP Projection vs Layer Correlation}
    We look at the magnitude of the correlations of the hidden layer weights of the depth-two MLP with the support of a 100-dimensional 6-sparse parity after the first phase of training. We observe that the curriculum extraction in-support coverage is significantly larger the out-of-support coverage, and with a significantly larger advantage than that for progressive distillation.
    }
    \label{fig:mlp_proj_vs_layer_correlation}.
\end{figure*}

\end{document}